%% file: arxiv.tex
\documentclass{article}
\usepackage[utf8]{inputenc}
\usepackage{arxiv-macros}
\usepackage[numbers]{natbib}

\usepackage{geometry}

\renewcommand{\citet}{\cite}
\newif\ifarxiv
\arxivtrue

\title{Differentially Private Algorithms for the Stochastic Saddle Point Problem with Optimal Rates for the Strong Gap}

\author{
\makebox[1.0in]{\hfill Raef Bassily\thanks{Department of Computer Science \& Engineering and the Translational Data Analytics Institute (TDAI), The Ohio State University,  \href{mailto:bassily.1@osu.edu}{bassily.1@osu.edu} }}
\makebox[1.8in]{\hfill Crist\'obal Guzm\'an
\thanks{Institute for Mathematical and Computational Engineering, Faculty of Mathematics and School of Engineering, Pontificia Universidad Cat\'olica de Chile,
 \href{mailto:crguzmanp@mat.uc.cl}{crguzmanp@mat.uc.cl} }}
\and
\makebox[1.2in]{\hfill Michael Menart \thanks{Department of Computer Science \& Engineering, The Ohio State University,
\href{mailto:menart.2@osu.edu}{menart.2@osu.edu}}}}
\usepackage{times}

\newcommand{\rnote}[1]{}
\newcommand{\mnote}[1]{}
\newcommand{\cnote}[1]{}

\usepackage{enumitem}

\begin{document}

\maketitle

\input{COLT2023/Sections/prelim}

\input{COLT2023/Sections/recursive-regularization.tex}

\input{COLT2023/Sections/dp-rates.tex}

\input{COLT2023/Sections/strong-vs-weak.tex}

\section*{Acknowledgements}\label{sec:ack}

RB's and MM's research is supported by NSF CAREER Award 2144532 and NSF Award AF-1908281. CG's research was partially supported by
INRIA Associate Teams project, FONDECYT 1210362 grant, ANID Anillo ACT210005 grant, and National Center for Artificial Intelligence CENIA FB210017, Basal ANID.

\bibliographystyle{alpha}
\bibliography{refs,refs2}

\newpage
\input{COLT2023/Sections/appendix.tex}

\end{document}

%% file: COLT2023/Sections/prelim.tex
\begin{abstract}%
  We show that convex-concave Lipschitz stochastic saddle point problems (also known as stochastic minimax optimization) can be solved under the constraint of $(\epsilon,\delta)$-differential privacy with \emph{strong (primal-dual) gap} rate of $\tilde O\big(\frac{1}{\sqrt{n}} + \frac{\sqrt{d}}{n\epsilon}\big)$, where $n$ is the dataset size and $d$ is the dimension of the problem. This rate is nearly optimal, based on existing lower bounds in differentially private stochastic convex optimization. Specifically, we prove a tight upper bound on the strong gap via novel implementation and analysis of the recursive regularization technique repurposed for saddle point problems. We show that this rate can be attained with  $O\big(\min\big\{\frac{n^2\epsilon^{1.5}}{\sqrt{d}}, n^{3/2}\big\}\big)$ gradient complexity, and $\tilde{O}(n)$ gradient complexity if the loss function is smooth. As a byproduct of our method, we develop a general algorithm that, given a black-box access to a subroutine satisfying a certain $\alpha$ primal-dual accuracy guarantee with respect to the empirical objective, gives a solution to the stochastic saddle point problem with a strong gap of $\tilde{O}(\alpha+\frac{1}{\sqrt{n}})$. We show that this $\alpha$-accuracy condition is satisfied by standard algorithms for the empirical saddle point problem such as the proximal point method and the stochastic gradient descent ascent algorithm. Finally, to emphasize the importance of the strong gap as a convergence criterion compared to the weaker notion of primal-dual gap, commonly known as the \emph{weak gap}, we show that even for simple problems it is possible for an algorithm to have zero weak gap and suffer from $\Omega(1)$ strong gap. We also show that there exists a fundamental tradeoff between stability and accuracy. Specifically, we show that any $\Delta$-stable algorithm has empirical gap $\Omega\big(\frac{1}{\Delta n}\big)$, and that this bound is tight. This result also holds also more specifically for empirical risk minimization problems and may be of independent interest.
\end{abstract}

\section{Introduction}
Stochastic (convex-concave) saddle point problems (SSP)\footnote{In this work, we will exclusively focus on the case where the function of interest for the stochastic saddle-point problem is convex-concave, and therefore we will omit it from the problem denomination.} (also referred to in the literature as stochastic minimax optimization problems) are an increasingly important model for modern machine learning, arising in areas such as stochastic optimization \citep{NJLS09,Juditsky:2011,Zhang:2015stochastic}, robust statistics  \citep{Yu:2021}, %
and algorithmic fairness \citep{pmlr-v97-mohri19a,Williamson:2019}. %

On the other hand, the reliance of modern machine learning on large datasets has led to concerns of user privacy. These concerns in turn have led to a variety of privacy standards, of which differential privacy (DP) has become the premier standard. However, for a variety of machine learning problems it is known that their differentially-private counterparts have provably worse rates. As such, characterizing the fundamental cost of differential privacy has become an important problem.

Currently, the theory of solving SSPs under differential privacy has major limitations, compared to its non-private counterpart. To illustrate this point, we need to discuss the notions of accuracy used in the literature.
In SSPs, the goal is to find an approximate solution of the problem
\begin{equation}\label{eqn:SSP}
\min_{w\in\cW} \max_{\theta\in\Theta} \Big\{ F_{\cal D}(w,\theta):= \mathbb{E}_{x\sim{\cal D}}[f(w,\theta;x)] \Big\}, 
\end{equation}
where ${\cal D}$ is an unknown distribution for which we have access to an i.i.d.~sample $S$. Given a (randomized) algorithm ${\cal A}$ with output $[{\cal A}_w(S),{\cal A}_{\theta}(S)] \in \cW\times\Theta$,
two studied measures of performance are the {\em strong and weak gap}\footnote{The weak gap is sometimes stated with $\mathbb{E}_{\cA}[\cdot]$ taken inside the max. However \cite{BG22} showed this was not necessary to obtain the stability implies generalization result used in various works.}, defined respectively as
\begin{eqnarray}
    \gap(\cA) &=& \ex{\cA,S}{\max_{\theta\in\Theta}\bc{F_{\cD}(\cA_w(S),\theta)} - \min_{w\in\cW}\bc{{F_{\cD}(w,\cA_{\theta}(S))}} }, \label{eqn:strong_gap} \\
    \weakgap(\cA) &=& \ex{\cA}{\max_{\theta\in\Theta}\bc{\ex{S}{F_{\cD}(\cA_w(S),\theta)}} - \min_{w\in\cW}\bc{\ex{S}{{F_{\cD}(w,\cA_{\theta}(S))}}}}. \label{eqn:weak_gap} 
\end{eqnarray}
 
It is easy to see that the strong gap upper bounds the weak gap, and thus it is a stronger accuracy measure. On the other hand, even for simple problems, the difference between these measures can be $\Omega(1)$; a fact we elaborate on in Section \ref{sec:strong-vs-weak}. We also note that the strong gap has a clear game-theoretic interpretation: if we consider ${\cal A}_w(S)$ and ${\cal A}_{\theta}(S)$ as the actions of two players in a (stochastic) zero-sum game, the strong gap upper bounds the most profitable unilateral deviation for either of the two players. In game theory this is known as an approximate Nash equilibrium. By contrast, there is no general guarantee associated with the weak gap.

Non-privately, it is known how to achieve optimal rates w.r.t.~the strong gap, and those rates are similar to those established for stochastic convex optimization (SCO) \citep{NJLS09,Juditsky:2011}. However, for DP methods optimal rates are only known for the weak gap \citep{BG22,yang-dp-sgda,ZTOH22}. In a nutshell, the main limitation of these approaches is that --in order to amplify privacy-- they make multiple passes over the data (e.g., by sampling with replacement stochastic gradients from the dataset), and the existing theory of generalization for SSPs is much more limited than it is for SCO \citep{zhang_generalization_saddle,lei-stability-generalization-minimax,OPZZ22}. %
Our approach largely circumvents the current limitations of generalization theory for SSPs, providing the first nearly-optimal rates for the strong gap in DP-SSP. %

\subsection{Contributions}
In this work, we establish the optimal rates on the strong gap for DP-SSP. %
In the following, we let $n$ be the number of samples, $d$ be the dimension, and $\epsilon,\delta$ be the privacy parameters. Our main result is an $(\epsilon,\delta)$-DP algorithm for SSP whose strong gap is $\tilde O\big(\frac{1}{\sqrt{n}} + \frac{\sqrt{d}}{n\epsilon}\big)$.  
This rate is nearly optimal, due to
matching lower bounds for differentially private SCO 
\citep{BST14,BFTT19}. %
These minimization lower bounds hold for saddle point problems since minimization problems are a special case of saddle point problems when $\Theta$ is constrained to be a singleton. 
For non-smooth loss function, we show this rate can be obtained in gradient complexity $O\big(\min\big\{\frac{n^2\epsilon^{1.5}}{\sqrt{d}}, n^{3/2}\big\}\big)$. This improves even upon the previous best known running time for achieving analogous rates on the \emph{weak gap}, which was $n^{5/2}$ \citep{yang-dp-sgda}.
Furthermore, we show that if the loss function is smooth, this rate can be achieved in nearly linear gradient complexity.  

In order to obtain an upper bound for this problem, we present a novel analysis of the recursive regularization 
algorithm of %
\cite{Allen-zhu-2018}. Our work is the first to show how the sequential regularization approach can be repurposed to provide an algorithmic framework for attaining optimal strong gap guarantees for DP-SSP.
As a byproduct of our analysis, we show that empirical saddle point solvers which satisfy a certain $\alpha$ accuracy guarantee can be used as a black box to obtain an $\tilde{O}\br{\alpha +1/\sqrt{n}}$ guarantee on the strong (population) gap. This class of algorithms includes common techniques such as the proximal point method, the extragradient method, %
and stochastic gradient descent ascent (SGDA) \citep{MOP20,Nemirovski-2004,Juditsky:2011}. 
This fact may be of interest independent of differential privacy, as to the best of our knowledge, existing algorithms which achieve the optimal $1/{\sqrt{n}}$ rate on the strong population gap rely crucially on a one-pass structure which optimizes the population gap directly \citep{NJLS09}. %

Under the additional assumption that the loss function is smooth, we show that it is possible to use recursive regularization to obtain the optimal strong gap rate in nearly linear time.
We here leverage accelerated algorithms for smooth and strongly convex/strongly concave loss functions \citep{PB16,JST22}. %

Our results stand in contrast to previous work on DP-SSPs, which has achieved optimal rates only for the weak gap and has crucially relied on ``stability implies generalization'' results for the weak gap. In this vein, we prove that even for simple problems, the strong and weak gap may differ by $\Theta(1)$. We also elucidate the challenges of extending existing techniques to strong gap guarantees by showing a fundamental tradeoff between stability and empirical accuracy. Specifically, we show that even for the more specific case of empirical risk minimization, any algorithm which is $\Delta$-uniform argument stable  %
algorithm must have empirical risk $\Omega\br{\frac{1}{\Delta n}}$. We also show this bound is tight, and note that it may be of independent interest. 
Such a tradeoff was also investigated by \cite{chen-stability-convergence}, but their result only implies such a tradeoff for the specific case of $\Delta=\frac{1}{\sqrt{n}}$ and their proof technique is unrelated to ours.

\subsection{Related Work}
Differentially private stochastic optimization has been extensively studied for over a decade \citep{jain2012differentially, BST14,JTOpt13, TTZ15a, BFTT19, feldman2020private, AFKT21, bassily2021non}. Among such problems, stochastic convex minimization (where problem parameters are measured in the $\ell_2$-norm) %
is perhaps the most widely studied, where it is known the optimal rate is $\tilde{O}(\frac{1}{\sqrt{n}} + \frac{\sqrt{d}}{n\epsilon})$ \citep{BFTT19,BST14}. Further, under smoothness assumptions such rates can be obtained in linear (in the sample size) gradient complexity \citep{FKT20}. Without smoothness, no linear time algorithms which achieve the optimal rates are known \citep{KLL21}.  

The study of stochastic saddle point problems under differential privacy is comparatively newer. In the non-private setting, optimal $O(1/\sqrt n)$ guarantees on the strong gap have been known as far back as \cite{Nemirovski:1978}. 
Under privacy (without strong convexity/strong concavity), optimal rates are known only for the \emph{weak gap}. 
These rates $\tilde{O}(\frac{1}{\sqrt{n}} + \frac{\sqrt{d}}{n\epsilon})$ have been obtained by several works \citep{BG22,yang-dp-sgda,ZTOH22}. The work of \cite{ZTOH22} additionally showed that under smoothness assumptions such a result could be obtained in near linear gradient complexity by leveraging accelerated methods \citep{JST22,PB16}. 
All of these results are for the weak gap 
and they rely crucially on the fact that, for the weak gap, $\Delta$-stability implies $\Delta$-generalization \cite{zhang_generalization_saddle}. %

By contrast, for the strong gap (without strong convexity/strong concavity assumptions), the best stability implies generalization result is a $\sqrt{\Delta}$ bound obtained by \cite{OPZZ22} provided the loss is smooth. As a result of this discrepancy, known bounds on the strong gap under privacy are worse. The best known rates for the strong gap are $O\br{\min\br{\frac{d^{1/4}}{\sqrt{n\epsilon}}, \frac{1}{n^{1/3}} + \frac{\sqrt{d}}{n^{2/3}\epsilon} }}$ \citep{BG22}. This rate was obtained through of mixture of noisy stochastic extragradient and noisy inexact proximal point methods, avoiding stability arguments altogether and instead relying on one-pass algorithms which optimize the population loss directly. %
Without smoothness, we are not aware of any work which provides bounds on the strong gap under privacy, but one may note that a straightforward implementation of one-pass noisy SGDA leads to a rate of $O\big( \frac{\sqrt{d}}{\sqrt{n}\epsilon}\big)$ in this setting. We give these details in Appendix \ref{app:local-privacy} and note this same algorithm establishes the optimal rate for SSPs under local differential privacy.

Finally, under the stringent assumptions of $\mu$-strong convexity/strong concavity ($\mu$-SC/SC)
and smoothness with constant condition number, $\kappa$, optimal rates on the strong gap have been obtained \citep{ZTOH22}. %
Under these assumptions, the optimal rate of $O\big(\frac{1}{\mu n} + \frac{d}{\mu n^2\epsilon^2}\big)$ was achieved by leveraging the fact that $\Delta$ stability implies $\kappa\Delta$ generalization \cite{zhang_generalization_saddle}. The lower bound for this rate comes from lower bounds for the minimization setting \citep{HK14,BFTT19}. 

\section{Preliminaries}
Throughout, we consider the space $\mathbb{R}^d$ endowed with the standard $\ell_2$ norm $\|\cdot\|$. 
Let the primal parameter space $\cW$ and the dual parameter space $\Theta$ be compact convex sets such that $\cW \times \Theta \subset \re^d$ for some $d>0$. Let $\cD$ be some distribution over data domain $\cX$. Consider the {\em stochastic saddle-point problem} given in equation~\eqref{eqn:SSP} %
for some loss function $f$ that is convex w.r.t.~$w$ and concave w.r.t.~$\theta$.  
We define the corresponding population loss and empirical loss functions as $F_{\cD}(w,\theta)=\ex{x\sim\cD}{f(w,\theta;x)}$ and $F_S(w,\theta)=\frac{1}{n}\sum_{x\in S} f(w,\theta;x)$ respectively.
For some $\rad > 0$ we assume that $\max_{u,u'\in\cW\times\Theta}\norm{u-u'} \leq \rad$. %
To simplify notation, for vectors $w\in\cW$ and $\theta\in\Theta$, we will use $[w,\theta]$ to denote their concatenation, noting $[w,\theta]$ is a vector in $\re^d$. 
We primarily consider the case where $f$ is $\lip$-Lipschitz, but will also consider the additional assumption of $\beta$-smoothness for certain results\footnote{Throughout, any properties for $f$ are considered as a function of $[w,\theta]$. No assumptions about $f$ w.r.t.~$x$ are made.}. Specifically, these assumptions are that $\forall w_1,w_2\in\cW$ and $\forall\theta_1,\theta_2\in\Theta$:
\begin{align*}
    &\mbox{Lipschitzness:}&&|f(w_1,\theta_1;x)- f(w_2,\theta_2;x)| \leq L\norm{[w_1,\theta_1] - [w_2,\theta_2]} 
    \\
    &\mbox{Smoothness:}&&\norm{\nabla_{[w,\theta]} f(w_1,\theta_1;x)-\nabla_{[w,\theta]} f(w_2,\theta_2;x)} \leq \beta\norm{[w_1,\theta_1] - [w_2,\theta_2]}. 
\end{align*}
Under such assumptions (in fact, smoothness is not necessary), a solution for problem \eqref{eqn:SSP} always exists \citep{Sion:1958}, which we will call as a {\em saddle point} onwards. Further, given an SSP \eqref{eqn:SSP}, we will denote a saddle point as $[w^{\ast},\theta^{\ast}]$. %

\paragraph{Gap functions}
In addition to the strong and weak gap functions defined in equations \eqref{eqn:strong_gap} and \eqref{eqn:weak_gap}, it will be useful to define the following {\em gap function} expressed as a function of the parameter vector instead of the algorithm, %
$\gapfunc(\bar{w},\bar{\theta}) = \max_{\theta\in\Theta}\bc{F_{\cD}(\bar w,\theta)} - \min_{w\in\cW}\bc{{F_{\cD}(w,\bar \theta)}}.$

We have the following useful fact regarding $\gapfunc$ (see Appendix \ref{app:prelim} for a proof).
\begin{fact}\label{fact:gap-lipschitz}
If $f$ is $\lip$-Lipschitz then $\gapfunc$ is $\sqrt{2}\lip$-Lipschitz. %
\end{fact}

Note the strong gap can be written as an expectation of the gap function. Further, since the gap function is zero if and only if $(\bar{w},\bar{\theta})$ is a solution for problem \eqref{eqn:SSP}, the strong gap is considered the most suitable measure of accuracy for SSPs \citep{Nemirovski:2010, Juditsky:2011}.
We also define the empirical gap as,
$\egap(\cA) = \ex{\cA}{\max_{\theta\in\Theta}\bc{F_{S}(\cA_w(S),\theta)} - \min_{w\in\cW}\bc{{F_{S}(w,\cA_{\theta}(S))}} }.$
We will consider at various points the notion of \textit{generalization error} with respect to the strong/weak gap, which refers to difference between the strong/weak gap and the empirical gap.
Note that because the empirical gap treats the dataset as a fixed quantity, there are not differing strong and weak versions of the empirical gap.

\paragraph{Saddle Operator}
Define the {\em saddle operator} as 
$g(w,\theta;x) = [\nabla_w f(w,\theta;x), -\nabla_\theta f(w,\theta;x)].$
Similarly define $G_{\cD}(w,\theta)=\mathbb{E}_{x\sim\cD}[g(w,\theta;x)]$ %
and $G_S(w,\theta) = \frac{1}{n}\sum_{x\in S} g(w,\theta;x)$. Note that the assumption on the smoothness of $f$ implies the Lipschitzness of $g$. We note that since the saddle operator can be computed using one computation of the gradient, we refer indistinctly to saddle operator complexity or gradient complexity when discussing the running time of our algorithms.

\paragraph{Stability}
We will also use the notion of  %
uniform argument stability frequently in our analysis \citep{bousquet2002stability}.%
\begin{definition}\label{def:uas}
A randomized algorithm $\cA:\cX^n\mapsto\cW\times\Theta$ satisfies $\Delta$-uniform argument stability if for any pair of adjacent datasets $S,S'\in\cX^n$ it holds that  
$\ex{\cA}{\norm{\cA(S)-\cA(S')}} \leq \Delta$. 
\end{definition}

A fact we will use is that the (constrained) regularized saddle-point is stable. Specifically, for some $\hat{w}\in\cW$, $\hat{\theta}\in\Theta$, and $\lambda\geq0$ consider the regularized objective function %
\begin{align} \label{eq:reg-erm}    
(w,\theta) \mapsto \frac{1}{n}\sum_{z\in S} f(w,\theta;z) + \frac{\lambda}{2}\|w-\hat w\|^2 - \frac{\lambda}{2}\|\theta - \hat \theta\|^2.
\end{align}
It is easy to see that his problem has a unique saddle point. 
The mapping which selects its output according the unique solution of \eqref{eq:reg-erm} %
has the following stability property. %
\begin{lemma}\citep[Lemma 1]{zhang_generalization_saddle} \label{lem:PPM-sensitivity}
The algorithm which outputs the regularized saddle point with parameters $\lambda>0$, $\hat{w}\in\cW$ and $\hat{\theta}\in\Theta$, is $\big(\frac{2 \lip }{\lambda n}\big)$-uniform argument stable w.r.t.~$S$. 
\end{lemma} 

In addition to the stability of the regularized saddle point, we will also frequently use the following fact.
\begin{lemma}\citep[Theorem 1]{zhang_generalization_saddle}
\label{lem:sc-sc-distance}
Let $h:\cW\times\Theta\mapsto\re$ be $\lambda$-SC/SC with saddle point $[w^*,\theta^*]$ and gap function $\gapfunc^h$. For any $[w,\theta]\in\cW\times\Theta$ it holds that $\norm{[w,\theta] - [w^*,\theta^*]}^2 \leq \frac{2(h(w,\theta^*) - h(w^*,\theta))}{\lambda} \leq \frac{2}{\lambda}\gapfunc^{h}(w,\theta)$. 
\end{lemma}

\paragraph{Differential Privacy (DP)~\cite{dwork2006calibrating}:}
An algorithm $\cA$ is $(\epsilon,\delta)$-differentially private if for all datasets $S$ and $S'$ differing in one data point and all events $\cE$ in the range of the $\cA$, we have, $\mathbb{P}\br{\cA(S)\in \cE} \leq     e^\epsilon \mathbb{P}\br{\cA(S')\in \cE}  +\delta$. 

%% file: COLT2023/Sections/recursive-regularization.tex
\section{From Empirical Saddle Point to Strong Gap Guarantee via Recursive Regularization} %
\label{sec:recursive-regularization}%
Our approach for obtaining near optimal rates on the strong gap leverages the recursive regularization technique of \cite{Allen-zhu-2018}. In addition to adapting this algorithm to fit SSP problems, we also provide a novel analysis which differs substantially from the analysis presented in previous work \citep{foster-complexity-gradient-small,arora2022differentially}.

\begin{algorithm}[h]
\caption{Recursive Regularization: $\cR$}
\label{alg:recursive-regularization}
\begin{algorithmic}[1]
\REQUIRE Dataset $S\in \cX^n$, 
loss function $f$,
subroutine $\weakalg$, %
regularization parameter $\lambda \geq \frac{\lip}{\rad\sqrt{n}}$, constraint set diameter $\rad$, Lipschitz constant $\lip$.

\STATE Let $n' = n/\log_2(n),$ and $T = \log_2(\frac{\lip}{\rad\lambda}).$ %
\STATE Let $S_1,...,S_T$ be a disjoint partition of $S$ with each $S_t$ of size $n'$ \textit{(which is always possible due to the condition on $\lambda$)}%
\STATE Let $[\bar{w}_0,\bar{\theta}_0]$ be any point in $\cW\times\Theta$

\STATE Define function $(w,\theta,x)\mapsto f^{(1)}(w,\theta;x) = f(w,\theta;x) + 2\lambda\norm{w-\bar{w}_0}^2 - 2\lambda\norm{\theta-\bar{\theta}_0}^2$
\FOR{$t=1$ to $T$} %

\STATE $[\bar{w}_t,\bar{\theta}_t] = \weakalg\br{S_t,f^{t},[\bar{w}_{t-1},\bar{\theta}_{t-1}],\frac{\rad}{2^t}}$

\STATE Define  $(w,\theta,x)\mapsto f^{(t+1)}(w,\theta;x) = f^{(t)}(w,\theta;x) + 2^{t+1}\lambda\norm{w-\out_t}^2 - 2^{t+1}\lambda\norm{\theta-\bar{\theta}_t}^2$
\ENDFOR
\STATE \textbf{Output:} {$[\bar{w}_{T},\bar{\theta}_T]$}
\end{algorithmic}
\end{algorithm}

Our recursive regularization algorithm works by solving a series of regularized objectives, $f^{(1)},...,f^{(T)}$, with increasingly large regularization parameters. Specifically, after solving the $t$'th objective to obtain $[\bar{w}_t,\bar{\theta}_t]$, the algorithm creates a new objective which is $f^{(t+1)}(w,\theta;x) = f^{(t)}(w,\theta;x) + 2^{t+1}\lambda\norm{w-\out_t}^2 - 2^{t+1}\lambda\norm{\theta-\bar{\theta}_t}^2$ for the subsequent round. %
Notice 
that each subsequent objective is easier in the sense that the strong convexity parameter is larger. 

Our analysis will leverage the fact that approximate solutions to intermediate objectives do not need to obtain good bounds on the strong gap for the regularization parameter to be increased. This is in contrast to, for example, the \emph{iterative} regularization technique of \cite{ZTOH22}, which finds $[w,\theta]$ that satisfies a near optimal (weak) gap bound before adding noise. %

\paragraph{Empirical Subroutine}
Recursive regularization utilizes a subroutine, $\weakalg$, which is roughly an approximate empirical saddle point solver. In addition to a dataset and Lipschitz loss function, $\weakalg$ takes as input an initial point and a bound, $\hat{D}$, on the expected distance between the initial point and the %
saddle point of the empirical loss defined over the input dataset.
At round $t\in[T]$ this distance is bounded by $\frac{\rad}{2^t}$, allowing the algorithm to obtain increasingly strong accuracy guarantees for each subproblem. Note also it can be verified that for all $t\in[T]$, $f^{(t)}$ is $O(\lip)$-Lipschitz due the scaling of the regularization.
Specifically, the accuracy guarantee of interest is the following. 
\begin{definition}[$\alphad$-relative accuracy] \label{asm:weak-alg}
Given a dataset $S'\in\cX^{n'}$, loss function $f'$, and an initial point $[w',\theta'],$ we say that $\weakalg$ satisfies $\alphad$-relative accuracy w.r.t.~the empirical saddle point $[w_{S'}^*, \theta_{S'}^*]$ of $F'_{S'}(w,\theta)=\frac{1}{n}\sum_{x\in {S'}}f'(w,\theta;x)$ if, $\forall \hat{D}>0$, whenever $\ex{}{\norm{[w',\theta'] - [w^*_{S'},\theta^*_{S'}]}} \leq \hat{D}$, the output $[\bar{w},\bar{\theta}]$ of $\weakalg$ satisfies $\ex{}{F'_{S'}(\bar{w},\theta^*_{S'})- F'_{S'}(w^*_{S'},\bar{\theta})} \leq \hat{D}\alphad$. 
\end{definition}
The relative accuracy guarantee for $\weakalg$ differs from the more standard gap guarantee, and is not necessarily implied by a bound on the empirical gap. 
The motivation for this notion of accuracy is twofold. First, when the loss function is additionally SC/SC, this guarantee is sufficient to provide a bound on the distance between the \textit{output} of $\weakalg$ and the saddle point, which will play a crucial role in our convergence proof for Algorithm \ref{alg:recursive-regularization}.
Second, while it is certainly true that a bound on the empirical gap implies the same bound on $\ex{}{F_S(\bar{w},\theta)- F_S(w,\bar{\theta})}$, for any given $[w,\theta]$, it is not necessarily the case that the gap itself may enjoy a bound that is proportional to the initial distance to the saddle point\footnote{\cite[Theorem 4]{FO20} claims such a bound on the primal risk, but this is due to a misapplication of \cite[Lemma 2]{MOP20}.}. The reason is that the gap function is defined by a supremum that is taken w.r.t.~the whole feasible set $\cW\times\Theta$, and thus the information of the evaluation of the objective w.r.t.~particular points is lost.
However, it is usually the case that saddle point solvers  provide a bound of the form $F_S(\bar{w},\theta) - F_S(w,\bar{\theta}) \leq \norm{[w,\theta] - [w',\theta']}\hat{\alpha}$, for all $[w,\theta]\in\cW\times\Theta$, and some initial point $[w',\theta']\in\cW\times\Theta$. Algorithms such as the proximal point method, extragradient method, and SGDA (with appropriately tuned learning rate) satisfy this condition, and thus satisfy the condition for relative accuracy \citep{MOP20,Nemirovski-2004,Juditsky:2011}.

\paragraph{Guarantees of Recursive Regularization}
Given such an algorithm, recursive regularization achieves the following guarantee.
\begin{theorem} \label{thm:nonsmooth-minimax-alg}
Let $\weakalg$ satisfy $\alphad$-relative accuracy for any 
$(5\lip)$-Lipschitz 
loss function and dataset of size $n'=\frac{n}{\log(n)}$.
Then 
Algorithm~\ref{alg:recursive-regularization}, 
run with $\weakalg$ as a subroutine and $\lambda=\frac{48}{\rad}\br{\alphad + \frac{\lip}{\sqrt{n'}}}$,
satisfies 
\begin{align*}
    \gap(\cR) = O\br{\log(n)\rad\alphad + \frac{\log^{3/2}(n)\rad\lip}{\sqrt{n}}}.
\end{align*}
\end{theorem}

Recall that $\rad$ is a bound on the diameter of the constraint set. %
In the following, we will sketch the proof of this theorem and highlight key lemmas. We defer the full proof to Appendix \ref{app:recursive-regularization-proof}. For simplicity, let us here consider the case where $\hat{\alpha}=0$. 
A crucial aspect of our proof is that we avoid the need to bound the strong gap of the actual iterates, $\bc{\bar{w}_t}_{t=1}^{T-1}$. Instead, we bound the strong gap of the \emph{expected} iterates, where the expectation is taken with respect to $S_t$. 
More concretely, consider some $t\in [T]$ and let $\cB$ be the algorithm which on input $[\bar{w}_{t-1},\bar{\theta}_{t-1}]$ outputs %
$\ex{S_t,\weakalg}{\weakalg(S_t,f^t,[\bar{w}_{t-1},\bar{\theta}_{t-1}],\frac{\rad}{2^t})}$. Note $\cB$ is deterministic and data independent. As a result, it is possible to prove bounds on the strong gap of $\cB$.
\begin{lemma} \label{lem:weak-bounds-expected}
Let $S\sim\cD^n$. For any $\Delta$-uniform argument stable algorithm $\cA$, 
it holds that
    $$\gapfunc\br{\ex{\cA,S}{\cA_w(S)},\ex{\cA,S}{\cA_\theta(S)}} \leq \weakgap(\cA) {\leq} \ex{S}{\egap(\cA)} + \Delta\lip.$$
\end{lemma}
The proof follows straightforwardly from an application of Jensen's inequality and the ``stability implies generalization'' result for the weak gap \cite[Theorem 1]{lei-stability-generalization-minimax}. We give full details in Appendix \ref{app:weak-bounds-expected-proof}. Note that, for this discussion, the LHS of the above is equal to $\gap(\cB)$ when we apply this lemma to the data batch $S_t$ and subroutine $\weakalg$. %

In fact, running $\cB$ is infeasible. Instead, we show that the output $\weakalg$ is close to the output of $\cB$. %
This in turn can be accomplished using the fact that bounded stability implies bounded variance. Concretely, we use the vector valued version of McDiarmid's inequality. 
\begin{lemma}\citep[Lemma 6]{rivasplata_pac_for_stable} \footnote{Although stated therein for the distance, the last step of their proof shows a squared distance bound can be obtained.} \label{lem:stability-implies-variance}
Let $\cA$ be deterministic $\Delta$-uniform argument stable 
stable with respect to $S\sim \cD^n$. Then its output satisfies 
$\mathbb{E}\Big[\big\| \cA(S) - \mathbb{E}_{\hat{S}\sim\cD^n}\big[\cA(\hat{S})\big] \big\|^2 \Big] 
    \leq n\Delta^2.$
\end{lemma}

Observe that the exact empirical saddle point is a deterministic quantity conditioned on the randomness of the $t$'th empirical objective.
Using the fact that $(2^{t}\lambda)$-regularization implies $\big(\frac{\lip}{2^t\lambda n'}\big)$-stability of the empirical saddle point in conjunction with the above lemma, we obtain a (conditional) variance bound of $\frac{\lip^2}{2^{2t}\lambda^2n'}$. 
Under the setting of $\lambda = \Omega(\frac{\lip}{\rad\sqrt{n'}})$, we can ultimately prove that the distance between the output of $\weakalg$ and $\cB$ (at round $t$) is  $O(\frac{\rad}{2^t})$. 
Since the strong gap of $\cB$ with respect to $F_{\cD}^{(t)}(w,\theta):=\mathbb{E}_{x\sim{\cal D}}[f^{(t)}(w,\theta;x)]$ 
is at most $\Delta\lip=\frac{\lip^2}{2^t\lambda n'}$ by Lemma \ref{lem:weak-bounds-expected} (recall we here assume $\alphad=0$ for simplicity) %
and $F_{\cD}^{(t)}$ is  $(2^{t+1}\lambda)$-SC/SC, the output of $\cB$ must in turn be close to the population saddle point. Specifically, this distance is also bounded as %
$%
\big(\frac{\Delta\lip}{2^{t}\lambda}\big)^{1/2}
= \frac{\lip}{\sqrt{2^t\lambda n'}}\frac{1}{\sqrt{2^t\lambda}} = O(\frac{\rad}{2^t})$. 
Thus we ultimately have that the distance between $[\bar{w}_t,\bar{\theta}_t]$ and the population saddle point of $F^{(t)}_{\cD}$, $[w^*_t,\theta^*_t]$, satisfies $\ex{}{\norm{[\bar{w}_t,\bar{\theta}_t]-[w^*_t,\theta^*_t]}} = O(\frac{\rad}{2^t})$. These ideas also lead to a bound %
$\ex{}{\norm{[w^*_{t+1},\theta^*_{t+1}]-[\bar{w}_t,\bar{\theta}_t]}} = O(\frac{\rad}{2^t})$, although the argument in this case is more technical and thus deferred to the full proof.

The upshot of this analysis is that as the level of regularization increases, the distance of the iterates to the their respective population minimizers decreases in kind. 
One consequence of this fact is that $\norm{[\bar{w}_T,\bar{\theta}_T]-[w_T^*,\theta_T^*]} = \tilde{O}\br{\frac{\rad}{\sqrt{n}}}$, and thus by the Lipschitzness of the gap function, the output of recursive regularization has a gap bound close to that of $[w^*_T,\theta_T^*$].
Turning now towards the utility of $[w^*_T,\theta_T^*]$, using the fact that $F_{\cD}$ is convex-concave we have 
\begin{align*}
    \gapfunc(w^*_T,\theta^*_T)
    &\leq \max\limits_{w'\in\cW,\theta'\in\Theta}\bc{\ip{G_{\cD}(w_T^*,\theta_T^*)}{[w^*_T,\theta^*_T] - [w',\theta']}}.  
\end{align*}
Further, an expression for $G_{\cD}$ be obtained using the definition of $F_{\cD}^{(T)}$:
\begin{align*}
    G_{\cD}(w_{T}^*,\theta_{T}^*)
    &= \textstyle G_{\cD}^{(T)}(w_{T}^*,\theta_{T}^*) - 2\lambda\sum_{t=0}^{T-1}
    2^{t+1}([w_{T}^*,-\theta_{T}^*] - [\bar{w}_t,-\bar{\theta}_t]), 
\end{align*}
where $G_{\cD}^{(T)}$ is the saddle operator of $F_\cD^{(T)}$. %
Plugging the latter into the former and using Cauchy-Schwarz inequality, the triangle inequality, %
and the fact that $[w_{T}^*,\theta_{T}^*]$ is the exact saddle point of $F_{\cD}^{(T)}$, 
one can obtain a bound on the gap in terms of the distances discussed previously. 
\begin{align*}
    &\ex{}{\gapfunc(w_{T}^*,\theta_{T}^*)} 
    \leq 4\rad\cdot\ex{}{\lambda \sum\limits_{t=0}^{T-1} 2^{t}\norm{[w_{T}^*,\theta_{T}^*] - [\bar{w}_t,\bar{\theta}_t]}} \\
    &\overset{(i)}{\leq} 4\rad\cdot\ex{}{\lambda \sum\limits_{t=0}^{T-1} 2^{t} \br{\norm{[w_{t+1}^*,\theta_{t+1}^*] - [\bar{w}_t,\bar{\theta}_t]} + \sum\limits_{r=t+1}^{T-1}\norm{[w_{r+1}^*,\theta_{r+1}^*]  - [w_{r}^*,\theta_{r}^*]}}} \\
    &\overset{(ii)}{=} O\br{\rad\sum\limits_{t=0}^{T-1} 2^{t}\lambda \ex{}{\norm{[w_{t+1}^*,\theta_{t+1}^*] - [\bar{w}_t,\bar{\theta}_t]}} + \rad\sum\limits_{t=1}^{T-1}2^t\lambda\ex{}{\norm{[\bar{w}_{t},\bar{\theta}_{t}] - [w_{t}^*,\theta_{t}^*]}}} \\
    &= O\br{\rad\sum\limits_{t=0}^{T-1} 2^{t}\lambda \frac{\rad}{2^{t}} + \rad\sum\limits_{r=1}^{T-1}2^{t}\lambda\frac{\rad}{2^{t}}}
    = O\br{T \lambda \rad^2}
    = O\br{\frac{\log_2(n)\rad\lip}{\sqrt{n'}}},
\end{align*}
where step $(i)$ comes from a triangle inequality and step $(ii)$ is obtained from a series of algebraic manipulations which are expanded upon in the full proof. Finally, in the case where $\hat{\alpha} > 0$, extra steps are required to bound the distance of output of $\weakalg$ to the exact saddle point of $F_{S}^{(t)}(w,\theta):=\frac{1}{n'}\sum_{x\in S_t}f^{(t)}(w,\theta;x)$. This is accomplished using the SC/SC property of $F_{S}^{(t)}$ and the $\alphad$-relative accuracy guarantee of $\weakalg$.

%% file: COLT2023/Sections/dp-rates.tex
\section{Optimal Strong Gap Rate for DP-SSP} \label{sec:dp-rates}
With the guarantees of recursive regularization established, what remains is to show there exist $(\epsilon,\delta)$-DP algorithms which achieve a sufficient accuracy on the empirical objective. Note this suffices to make the entire recursive regularization algorithm private. 
\begin{theorem}
Let $\weakalg$ used in Algorithm \ref{alg:recursive-regularization} be $(\epsilon,\delta)$-DP. Then Algorithm \ref{alg:recursive-regularization} is $(\epsilon,\delta)$-DP.
\end{theorem}
This follows simply from post processing the parallel composition theorem for differential privacy, since each run of $\weakalg$ is run on a disjoint partition of the dataset. %

\subsection{Efficient algorithm for the non-smooth setting}%
In the non-smooth setting, one can obtain optimal rates on the empirical gap using noisy stochastic gradient descent ascent (noisy SGDA). 
We give this algorithm in detail in Appendix \ref{app:noisy-sgda}. More briefly, noisy SGDA starts at  %
$[w_0,\theta_0]\in\cW\times\Theta$ and takes parameters $T,\eta>0$, where $T$ is the number of iterations and $\eta$ is the learning rate. New iterates are obtained via the update rule $[w_{t+1},\theta_{t+1}] = [w_{t},\theta_{t}] - \frac{\eta}{|M_t|}\sum_{x\in M_t}g(w_{t},\theta_{t};x) + \xi_t$, where $\xi_0,...,\xi_{T-1}$ are i.i.d. Gaussian noise vectors and $M_t$ is a minibatch sampled uniformly with replacement from $S$. The algorithm then returns the average iterate, $\frac{1}{T}\sum_{t=0}^{T-1}[w_t,\theta_t]$. Noisy SGDA can be used to obtain the following result. 
\begin{lemma}
There exists an $(\epsilon,\delta)$-DP algorithm which 
satisfies $\alphad$-relative accuracy
with \linebreak $\alphad=O\br{\frac{\log(n)\lip\sqrt{d\log(1/\delta)}}{n\epsilon}}$ and runs in $O\br{\min\bc{\frac{n^2\epsilon^{1.5}}{\log^2(n)\sqrt{d\log(1/\delta)}}, \frac{n^{3/2}}{\log^{3/2}(n)}}}$ gradient evaluations. 
\end{lemma}
Applying Theorem \ref{thm:nonsmooth-minimax-alg} then yields a near optimal rate on the strong gap.
\begin{corollary}
There exists an Algorithm, $\cR$, which is $(\epsilon,\delta)$-DP, has gradient evaluations bounded by \ifarxiv \linebreak \fi $O\big(\min\big\{\frac{n^2\epsilon^{1.5}}{\log(n)\sqrt{d\log(1/\delta)}}, \frac{n^{3/2}}{\sqrt{\log(n)}}\big\}\big)$, %
and satisfies
\begin{align*}
    \gap(\cR) = O\br{\frac{\log^{3/2}(n)\rad\lip}{\sqrt{n}} + \frac{\log^2(n)\rad\lip\sqrt{d\log(1/\delta)}}{n\epsilon}}.
\end{align*}
\end{corollary}

\subsection{Near linear time algorithm for the smooth setting}%

In the smooth setting, we can achieve the optimal rate in nearly linear time. Our result leverages accelerated algorithms for smooth and strongly convex-strongly concave saddle point problems \citep{JST22,PB16}. 

\begin{lemma}(\citet[Theorem 3, Corollary 41]{JST22}) \label{lem:acc-rate}
Let $f:\cW\times\Theta\times\cX\mapsto\re$ be $\smooth$-smooth and $\alpha > 0$. Let both $h_w:\cW\mapsto \re$ and $h_\theta:\Theta\mapsto \re$ be $c_1\mu$-strongly convex and $c_2\mu$-smooth 
functions for some $\mu > 0$ and constants $c_1,c_2$. Consider the objective $F_h(w,\theta;S) = \sum_{t=1}^T f(w,\theta;S) + h_w(w) - h_\theta(\theta)$. Then there exists an algorithm which finds an approximate saddle point of $F_h$ with empirical gap at most $\alpha$ in 
$O\br{\kappa \log(\kappa)\log(\frac{\kappa\rad\lip}{\alpha})}$ gradient evaluations, where $\kappa = O(n + \sqrt{n}(1 + \smooth/\mu))$.
\end{lemma}

Given this, we consider the following implementation of $\weakalg$. Define $[w_{S,t}^*,\theta_{S,t}^*]$ to be
the saddle point of $F^{(t)}(w,\theta) = \frac{1}{n}\sum_{x\in S_t}f^{(t)}(w,\theta;x)$ for all $t\in[T]$.
At round $t\in[T]$, find a point $[\hat{w}_t,\hat{\theta}_t]$ such that $\ex{}{\|[\hat{w}_t,\hat{\theta}_t] - [w_{S,t}^*,\theta_{S,t}^*]\|^2} \leq \big(\frac{\delta}{5}\cdot\frac{\lip}{2^t \lambda n'}\big)^2$. %
We can find this point efficiently using the algorithm from \cite{JST22} referenced above.
Then output 
$[\bar{w}_t,\bar{\theta}_t] = [\hat{w}_t,\hat{\theta}_t] + \xi_t$ where $\xi_t\sim\cN(0,\mathbb{I}_d\sigma_t^2)$ and 
$\sigma_t = \frac{8\lip\sqrt{\log(2/\delta)}}{2^t \lambda n'\epsilon}$. This implementation gives us the following result. 

\begin{theorem} \label{thm:smooth-minimax-alg}
    Let $\weakalg$ be as described above. Then Algorithm \ref{alg:recursive-regularization} is $(\epsilon,\delta)$-DP and when run with $\lambda = \frac{48}{\rad}\br{\frac{\lip}{\sqrt{n'}} + \frac{\lip\sqrt{d\log(2/\delta)}}{n'\epsilon}}$ satisfies
\begin{align*}
    \gap(\cR) = O\br{\frac{\log^{3/2}(n)\rad\lip}{\sqrt{n}} + \frac{\log^2(n)\rad\lip\sqrt{d\log(1/\delta)}}{n\epsilon}},
\end{align*}
and runs in at most $O\br{\kappa\log(\kappa)\log(\kappa n/\delta)\log(n)}$ gradient evaluations with $\kappa = O\br{n+n\beta\rad/\lip}$.
\end{theorem}
\begin{proof}[proof of Theorem \ref{thm:smooth-minimax-alg}]
In the following, we start by proving the privacy guarantee. Then, we prove the utility guarantee, and finish by verifying the running time of the algorithm.

\vspace{8pt}\noindent\textit{Privacy Guarantee:}
Consider any $t\in[T]$ and fix $[w_1,\theta_1],...,[w_{t-1},\theta_{t-1}]$.
The stability of the regularized saddle point at round $t$, $[w^*_{S,t},\theta^*_{S,t}]$,
is then $\frac{\lip}{2^t \lambda n'}$ by Lemma \ref{lem:PPM-sensitivity}. 
Since $\weakalg$ guarantees that $\ex{}{\|[\hat{w}_t,\hat{\theta}_t]-[w^*_{S,t},\theta^*_{S,t}]\|} \leq\frac{\delta}{5}\cdot\frac{\lip}{2^t \lambda n'}$, we have by Markov's inequality that with probability at least $1-\frac{\delta}{2}$ that $\|[\hat{w}_t,\hat{\theta}_t]-[w^*_{S,t},\theta^*_{S,t}]\| \leq \frac{\lip}{2^t \lambda n'}$. Thus with probability at least $1-\frac{\delta}{2}$, generating $[\hat{w}_t,\hat{\theta}_t]$ satisfies $\frac{2\lip}{2^t \lambda n'}$ uniform argument stability. %
Thus Gaussian noise of scale %
$\sigma_t = \frac{8\lip\sqrt{\log(2/\delta)}}{2^t \lambda n'\epsilon}$ ensures the round is $(\epsilon,\delta)$-DP. Parallel composition then ensures the entire algorithm is $(\epsilon,\delta)$-DP since each phase acts on a disjoint partition of the dataset.

\vspace{8pt}\noindent\textit{Utility Guarantee:} We now turn to the accuracy guarantee. %
Specifically, we leverage the generalized convergence guarantee of Algorithm \ref{alg:recursive-regularization} given by Theorem \ref{thm:generalized-rr-convergence} in Appendix \ref{app:recursive-regularization}. This theorem guarantees that so long as the distance condition 
$\ex{}{\norm{[\bar{w}_t,\bar{\theta}_t] - [w^{*}_{S,t},\theta^*_{S,t}]}^2} \leq \frac{\rad^2}{12\cdot2^{2t}}$
is satisfied for all $t\in[T]$, one obtains convergence guarantee $\gap(\cR) = O(\log(n)\rad^2\lambda)$. That is, after the distance guarantee is established, the rest of the analysis (i.e. the proof of Theorem \ref{thm:generalized-rr-convergence}) follows the same lines as in the non-smooth case. 
Note under the setting of $\lambda$ in Theorem \ref{thm:smooth-minimax-alg} we have
\begin{align*}
    \gap(\cR) = O(\log(n)\rad^2\lambda) = O\br{\frac{\log^{3/2}(n)\rad\lip}{\sqrt{n}} + \frac{\log^2(n)\rad\lip\sqrt{d\log(2/\delta)}}{n\epsilon}}.
\end{align*}

Thus all that remains is to show that the distance condition, $\ex{}{\norm{[\bar{w}_t,\bar{\theta}_t] - [w^{*}_{S,t},\theta^*_{S,t}]}^2} \leq \frac{\rad^2}{12\cdot2^{2t}}$, is satisfied for all $t\in[T]$.
In this regard we have,
\begin{align*}
    \ex{}{\norm{[\bar{w}_t,\bar{\theta}_t] - [w_{S,t}^*,\theta_{S,t}^*]}^2} 
    &\leq \ex{}{\|[\bar{w}_t,\bar{\theta}_t] - [\hat{w}_t,\hat{\theta}_t]\|^2 + \|[\hat{w}_t,\hat{\theta}_t] - [w_{S,t}^*,\theta_{S,t}^*]\|^2} \\
    &\leq d\sigma_t^2 + \br{\frac{\delta}{5}\cdot\frac{\lip}{2^t \lambda n'}}^2  \\
    &\leq \frac{64 d \lip^2\log(2/\delta)}{2^{2t}\lambda^2(n')^2\epsilon^2} + \frac{\rad^2}{25 \cdot 2^{2t}} 
    \leq \frac{\rad^2}{12 \cdot 2^{2t}}.
\end{align*}
For the first inequality, observe that the noise vector is uncorrelated with the vectors, $[\hat{w}_t,\hat{\theta}_t]$ and $[w_{S,t}^*,\theta_{S,t}^*]$. For the second inequality note $\ex{}{\|[\bar{w}_t,\bar{\theta}_t] - [\hat{w}_t,\hat{\theta}_t]\|^2}=\ex{}{\|\xi_t\|^2}=d\sigma^2_t$. 
Further, \ifarxiv \linebreak \fi $\ex{}{\|[\hat{w}_t,\hat{\theta}_t] - [w_{S,t}^*,\theta_{S,t}^*]\|^2}$ is bounded due to the chosen implementation of $\weakalg$.
The third inequality comes from the settings of $\sigma_t$ and the fact that $\lambda > \frac{48\lip}{B\sqrt{n'}}$. The last inequality uses the fact that $\lambda > \frac{48\lip\sqrt{d\log(2/\delta)}}{\rad n'\epsilon}$. %

\vspace{8pt}\noindent\textit{Running Time:} 
One can ensure that overall algorithm runs in nearly linear time by leveraging accelerated methods to find the point $[\hat{w},\hat{\theta}_t]$. 
The description of $\weakalg$ requires that at each phase $t\in[T]$, one has $\ex{}{\|[\hat{w}_t,\hat{\theta}_t] - [w_{S,t}^*,\theta_{S,t}^*]\|^2} \leq \big(\frac{\delta}{5}\cdot\frac{\lip}{2^t \lambda n'}\big)^2$, which by Lemma \ref{lem:sc-sc-distance} is satisfied if the empirical gap is at most $\lambda\big(\frac{\delta}{5}\cdot\frac{\lip}{2^t \lambda n'}\big)^2 = \frac{\delta^2}{25}\cdot\frac{\lip^2}{2^{2t}\lambda (n')^2}$.
For simplicity, we observe that%
\begin{align*}
    \frac{\delta^2}{25}\cdot\frac{\lip^2}{2^{2t}\lambda n'^2} = \Omega\br{\frac{\delta^2\lip^2}{2^{2T}\lambda (n')^2}} = \Omega\br{\frac{\rad^2\lambda^2}{\lip^2} \frac{\delta^2\lip^2}{\lambda (n')^2}} = \Omega\br{\frac{\delta^2\rad\lip}{n^{2.5}}}
\end{align*}
We now apply Lemma \ref{lem:acc-rate} with $h_w(w) = \lambda\sum_{k=0}^{t-1} 2^{k+1}\norm{w-\bar{w}_k}^2$, $h_\theta(\theta)=\lambda\sum_{k=0}^{t-1} 2^{k+1}\norm{w-\bar{w}_k}^2$, $\mu = 2^t\lambda$  
and $\alpha = \frac{c_3\delta^2\rad\lip}{n^{2.5}}$ for some sufficiently small constant $c_3$. This gives that the running time of phase $t$ is
$O\br{\kappa_t\log(\kappa_t)\log(\kappa_t n^{2.5}/\delta^2]}$, where
$\kappa_t = O\br{n + \sqrt{n}\smooth/[2^t\lambda])} = O\br{n+n\smooth\rad/\lip}$. Running this implementation of $\weakalg$ each phase incurs an extra factor of $T=\log(\frac{\lip}{\rad\lambda}) = O(\log(n))$, giving the claimed running time bound of $O\br{\kappa\log(\kappa)\log(\kappa n/\delta]\log(n)}$, where $\kappa = O\br{n+n\smooth\rad/\lip}$. 
\end{proof}

%% file: COLT2023/Sections/strong-vs-weak.tex
\section{On the Limitations of Previous Approaches} \label{sec:strong-vs-weak}
Prior work into DP SSPs has largely focused on the weak gap criteria.
In this section, we provide further investigation into both the importance and challenges of bounding the strong gap over the weak gap. We start by considering a natural question. Do there exist cases where the strong and weak gap differ substantially? We answer this question affirmatively in the following.
\begin{proposition}
There exists a convex-concave function $f$ with range $[-1,+1]$ 
and algorithm $\cA$ such that $\gap(\cA) - \weakgap(\cA)=2$.
\end{proposition}
Our construction shows that this result holds even for a simple one dimensional bilinear problem. 
\begin{proof}
Consider the loss function
$f(w,\theta;x) = w\theta$, where $w,\theta,x\in [-1,1].$
Let $\cD$ be the uniform distribution over $\bc{\pm 1}$. For $\bc{x_1,\dots,x_n}\sim\cD^n$ consider the algorithm $\cA$ which outputs $\bar w$ as the mode of the first half of the samples in $S$ and similarly $\bar \theta$ is set as the mode of the second half of the samples in $S$ \footnote{Without much loss of generality, we assume that $n$ is divisible by 2 but not by 4, so that the mode of each half of the data are well-defined and belong to $\{-1,+1\}$.}. Note $\bar{w}$ and $\bar{\theta}$ are independent and distributed uniformly over $\bc{\pm1}$ (under the randomness from $\cD$). 

Now, since $\cA$ is a deterministic function of the dataset, the randomness in $\bar{w},\bar{\theta}$ comes only from $S$. Thus for the weak gap we have 
$\max\limits_{\theta \in [-1,1]}\{\ex{S}{\bar w\theta}\} - \min\limits_{w\in[-1,1]}\{\ex{S}{w \bar \theta}\}$ which evaluates to $\max_{\theta \in [-1,1]}\{\ex{S}{\bar w}\theta\} - \min_{w\in[-1,1]}\{w\ex{S}{\bar \theta}\} = 0.$
However, one can see for the strong gap we have \linebreak %
$\ex{S}{\max\limits_{\theta \in [-1,1]}\bc{ \bar w\theta} - \min\limits_{w\in[-1,1]}\bc{w \bar \theta} }
    = \ex{S}{ \abs{\bar w} + \abs{\bar \theta}} = 2$,
where the first equality comes from evaluating $\theta=\mathsf{sgn}(\bar w)$ and $w=-\mathsf{sgn}(\bar \theta)$ in the maximization and minimization operators. 
\end{proof}

Observe that the generalization error w.r.t. the strong gap of this algorithm is always $0$ because the loss function does not depend on the random sample from $\cD$. The discrepancy between the gaps instead comes from the fact that having the expectation w.r.t. $S$ inside the max/min changes the function over which the dual/primal adversary is maximizing/minimizing. Specifically, note here that the weak gap measures the ability of $\theta$ to maximize the function $\theta \mapsto \bar{w}\theta$ for $\bar{w}=0$, but note $\bar{w}=0$ does not occur for \emph{any} realization of the dataset $S$.

One might further observe that a key attribute of this construction is the high variance of the parameter vectors. One can show such behavior is in fact necessary to see such a separation; the full proof of the following is statement is given in Appendix \ref{app:var-implies-strong-gap-proof}.%
\begin{proposition} \label{prop:var-implies-strong-gap}
Let $\cA$ be an algorithm such that 
$\ex{\cA,S}{\norm{\cA(S) - \E_{\hat{S}\sim\cD^n,\cA}{\cA(\hat S)}}^2} \leq \tau^2,$
then if $f$ is $\lip$-Lipschitz it holds that
$\gap(\cA) - \weakgap(\cA) \leq  \lip\tau.$
\end{proposition}

\paragraph{Tradeoff between Accuracy and Stability}
An additional consequence of Proposition \ref{prop:var-implies-strong-gap} (in conjunction with Lemma \ref{lem:stability-implies-variance}) is that $\Delta$-uniform argument stability implies $\sqrt{n}\Delta\lip$ generalization bound w.r.t. the strong gap that does not rely on smoothness (in contrast to the $\sqrt{\lip\beta\Delta}$ bound of \cite{OPZZ22} which does). We leave determining tight bounds for stability implies generalization on the strong gap as an interesting direction for future work. In this section however, we show that stronger upper bounds are likely necessary to obtain a more direct algorithm for DP-SSPs. In fact, our key result holds even for empirical risk minimization (ERM) problems. That is, for $f:\cW\times{\cal X}\mapsto\re$ and $S\in\cX^n$, consider the problem of minimizing the excess empirical risk $F_S(w) - \min_{w\in\cW}\bc{F_S(w)}$, where $F_S(w) = \frac{1}{n}\sum_{x\in S}f(w;x)$. We have the following. %
\begin{theorem}\label{thm:stab-risk-tradeoff}
For any (possibly randomized) algorithm $\cA:\cX^n\mapsto \cW$ which is $\Delta$-uniform argument stable, there exists a $0$-smooth $\lip$-Lipschitz loss function, $f:\cW\times \cX\mapsto\re$, and dataset $S\in\cX^n$ such that 
$\E[F_S(\cA(S)) - \min\limits_{w\in\cW}\bc{F_S(w)}] = \Omega\br{\frac{\rad^2\lip}{\Delta n}}$ provided $\Delta \geq \frac{\rad}{\sqrt{\min\bc{n,d}}}$.
\end{theorem}
The proof can be found in Appendix \ref{app:stability-risk-tradeoff}. Lemma \ref{lem:PPM-sensitivity} shows this bound is tight for both ERM and empirical saddle point problems. Generalization bounds are only useful when it is possible to obtain good empirical performance. Thus, the implication of this bound is that generalization error which is $O(\Delta)$ is necessary to obtain the optimal $O\br{1/\sqrt{n}}$ statistical rate. To elaborate, let $H(\Delta)$ characterize some (potentially suboptimal) generalization bound for $\Delta$ stable algorithms and assume $H(\Delta)=\omega(\Delta)$. To then bound the sum of empirical risk and generalization error, Theorem \ref{thm:stab-risk-tradeoff} implies
$F_S(\cA(S)) - F_S(w^*) + H(\Delta) = \Omega\br{ \frac{1}{\Delta n} + H(\Delta)} = \omega\br{\frac{1}{\Delta n} + \Delta}.$
Note the RHS is asymptotically larger than $\frac{1}{\sqrt{n}}$ (i.e. not optimal) for any $\Delta$.

%% file: COLT2023/Sections/appendix.tex
\appendix

\section{Supporting Proofs from Preliminaries} \label{app:prelim}
\subsection{Lipschitzness of the Gap Function}
\begin{proof}[proof of Fact \ref{fact:gap-lipschitz}]
For any $[\bar{w},\bar{\theta}],[\bar{w}',\bar{\theta}']\in\cW\times\Theta$ we have
\begin{eqnarray*}
\widehat{\gap}(\bar w,\bar\theta)-\widehat{\gap}(\bar w^{\prime},\bar\theta^{\prime})
&=&\sup_{w,\theta}\bc{ F_{\cD}(\bar w,\theta) - F_{\cD}(w,\bar\theta) } - \sup_{w,\theta}\bc{ F_{\cD}(\bar w^{\prime},\theta) - F_{\cD}(w,\bar\theta^{\prime}) }\\
&\leq&\sup_{w,\theta}\bc{ F_{\cD}(\bar w,\theta)  - F_{\cD}(\bar w^{\prime},\theta) + F_{\cD}(w,\bar\theta^{\prime}) - F_{\cD}(w,\bar\theta) }\\
&\leq& L \sup_{w,\theta}\bc{ \| \bar w  - \bar w^{\prime} \| + \|\bar\theta^{\prime} - \bar\theta \| }\\
&\leq& \sqrt{2}L\| [\bar w, \bar\theta] - [\bar w^{\prime},\bar\theta^{\prime}] \|,
\end{eqnarray*}
where we used in the last inequality that $a+b\leq \sqrt{2}\sqrt{a^2+b^2}$. %
\end{proof}

\subsection{Local Privacy } \label{app:local-privacy}
In the case of local differential privacy (LDP), a simple implementation of noisy SGDA (see Appendix \ref{app:sgda}) suffices to obtain the optimal rate. We defer the reader to \citet{DJW13} for a discussion of LDP and the matching lower bound.
Consider the implementation of SGDA which defines the saddle estimator as
\begin{align*}
    \nabla_t = g(w_{t-1},\theta_{t-1};x_t) + \xi_t
\end{align*}
where $\xi_t \sim \cN(0,\mathbb{I}_d\sigma)$ and 
$\sigma = \frac{\lip\sqrt{\log(1/\delta)}}{\epsilon}$ and $x_t$ is sampled without replacement from $S$.
By Lemma \ref{lem:noisy_SGDA} we have the following.
\begin{corollary}
Let $T=n$. Then the algorithm described above, denoted as $\cA$, is $(\epsilon,\delta)$-LDP and if $\eta = \frac{\rad}{\sqrt{nd\log(1/\delta)}\lip\epsilon}$ the average iterate, $[\bar w,\bar \theta]$, satisfies
$\gap(\cA) = O\br{\frac{\rad\lip\sqrt{d\log(1/\delta)}}{\sqrt{n}\epsilon}}.$
\end{corollary}

\input{COLT2023/Sections/app-recursive-regularization.tex}

\input{COLT2023/Sections/app-dp-rates.tex}

\input{COLT2023/Sections/app-strong-vs-weak.tex}

%% file: COLT2023/Sections/app-recursive-regularization.tex
\section{Missing Results from Section \ref{sec:recursive-regularization}} \label{app:recursive-regularization}

\subsection{Proof of Lemma \ref{lem:weak-bounds-expected}} \label{app:weak-bounds-expected-proof}
The first inequality follows from an application of Jensen's inequality.
\begin{align*}
&\gapfunc\br{\ex{\cA,S}{\cA_w(S)},\ex{\cA,S}{\cA_\theta(S)}}\\
&= \max_{\theta\in\Theta}\bc{F_{\cD}\Big(\ex{\hat{S}\sim\cD^n,\cA_w}{\cA_w(\hat S)},\theta\Big)} - \min_{w\in\cW}\bc{{F_{\cD}\Big(w,\ex{\hat{S}\sim\cD^n,\cA_\theta}{\cA_\theta(\hat S)}\Big)}} \nonumber \\
&\leq \max_{\theta\in\Theta}\bc{\ex{\hat{S}\sim\cD^n,\cA_w}{F_{\cD}(\cA_w(\hat S),\theta)}} - \min_{w\in\cW}\bc{\ex{\hat{S}\sim\cD^n,\cA_\theta}{{F_{\cD}(w,\cA_\theta(\hat S))}}} \nonumber \\
&= \weakgap(\cA). 
\end{align*}
The second inequality in the theorem statement then follows from stability implies generalization result for the weak gap, for which we provide a restatement below.
\begin{lemma}\cite[Theorem 1]{lei-stability-generalization-minimax}, \cite[Proposition 2.1]{BG22}
Let the loss function $f$ be $\lip$-Lipschitz and the algorithm $\cA$ be $\Delta$-uniform argument stable. Then 
$\weakgap(\cA) \leq \ex{S}{\egap(\cA)} + \Delta\lip.$
\end{lemma}

\subsection{Convergence of Recursive Regularization} \label{app:recursive-regularization-proof}
In this section we prove the following more general statement of Theorem \ref{thm:nonsmooth-minimax-alg}, which will be useful later. %
\begin{theorem}\label{thm:generalized-rr-convergence}
Let $\lambda \geq \frac{48\lip}{\rad\sqrt{n'}}$
and $\weakalg$ be such that for all $t\in [T]$ it holds that \linebreak $\ex{}{\norm{[\bar{w}_t,\bar{\theta}_t] - [w^{*}_{S,t},\theta^*_{S,t}]}^2} \leq \frac{\rad^2}{12\cdot2^{2t}}$.
Then Recursive Regularization satisfies
\begin{align*}
    \gap(\cR) = O\Big(\log(n)\rad^2\lambda\Big)
\end{align*}
\end{theorem}

To prove this result, it will be helpful to first show several intermediate results. 
We start by defining several useful quantities.
Define $\bc{\cF_t}_{t=0}^T$ as the filtration where $\cF_t$ is the sigma algebra induced by all randomness up to $[\bar{w}_t,\bar{\theta}_t]$.
For every $t\in[T]$ we define %
\begin{itemize}
    \item $[w^*_t,\theta^*_t]:$ saddle point of $F_\cD^{(t)}(w,\theta) := \ex{x\sim\cD}{f^{(t)}(w,\theta;x)}$;
    \item $[w^{*}_{S,t},\theta^*_{S,t}]:$ saddle point of $F_S^{(t)}(w,\theta):=\frac{1}{n}\sum_{x\in S}f^{(t)}(w,\theta;x)$;
    \item$[\widetilde{w}_t,\widetilde{\theta}_t] := \ex{}{[w^{*}_{S,t},\theta^*_{S,t}] \Big| \cF_{t-1}}$; %
    \item $\gapfunc^{(t)}(\bar{w},\bar{\theta}) :=\max\limits_{\theta\in\Theta}\bc{F^{(t)}_{\cD}(\bar{w},\theta)} - \min\limits_{w\in\cW}\bc{{F^{(t)}_{\cD}(w,\bar{\theta})}}:$ the gap function w.r.t.~$F_{\cD}^{(t)}$; and, 
    \item $\gapfunc_S^{(t)}(\bar{w},\bar{\theta}) :=\max\limits_{\theta\in\Theta}\bc{F^{(t)}_{S_t}(\bar{w},\theta)} - \min\limits_{w\in\cW}\bc{{F^{(t)}_{S_t}(w,\bar{\theta})}}:$ the empirical gap function. 
\end{itemize}

We now establish two distance inequalities which will be used when analyzing the final gap bound in Theorem \ref{thm:generalized-rr-convergence}.
The first inequality above bounds the distance of the output of the $t$-th round 
to the minimizer of $F_{\cD}^{(t)}$. The second inequality bounds the distance of the minimizer of $F_{\cD}^{(t)}$ 
to the most recent regularization point. 
\begin{lemma} \label{lem:phase-distance-bound}
Assume the conditions of Theorem \ref{thm:generalized-rr-convergence} hold. Then for every $t\in[T]$, the following holds
\begin{enumerate}[label=\textbf{P.\arabic*}]
    \item $\ex{}{\norm{[\bar{w}_t,\bar{\theta}_t] - [w_{t}^*,\theta_{t}^*]}}^2 \leq \ex{}{\norm{[\bar{w}_t,\bar{\theta}_t] - [w_{t}^*,\theta_{t}^*]}^2} \leq \frac{\rad^2}{2^{2t}}$; and, \label{prop:p1}
    \item $\rad_t^2 := \ex{}{\norm{[w_{t}^*,\theta_{t}^*] - [\bar{w}_{t-1},\bar{\theta}_{t-1}]}}^2 \leq \ex{}{\norm{[w_{t}^*,\theta_{t}^*] - [\bar{w}_{t-1},\bar{\theta}_{t-1}]}^2} \leq \frac{\rad^2}{2^{2(t-1)}}$. \label{prop:p2}
\end{enumerate}
\end{lemma}
\begin{proof}
We will prove both properties via induction on $\rad_1,...,\rad_T$. Specifically, for each $t\in[T]$ we will introduce three terms $E_t,F_t,G_t$, and show that 
these terms are bounded 
if the bound on $\rad_{t}$ holds and that $\rad_t$ holds if
$E_{t-1},F_{t-1},G_{t-1}$ are bounded. 
Property \ref{prop:p1} is then established as a result of the fact that $\ex{}{\norm{[\bar{w}_t,\bar{\theta}_t] - [w_{t}^*,\theta_{t}^*]}^2} \leq 3(E_t+F_t+G_t)$.
Note that $\rad_1$ holds as the base case because 
$\ex{}{\norm{[w_{1}^*,\theta_{1}^*] - [\bar{w}_{0},\bar{\theta}_{0}]}^2} \leq \rad^2$.%

\paragraph{Property \ref{prop:p1}:} 
We here prove that if $\rad_t$ is sufficiently bounded, then $E_t,F_t,G_t$ are bounded where 
for $t\in [T]$ we define
{\small \begin{align}
    E_t = \ex{}{\norm{[\bar{w}_t,\bar{\theta}_t] - [w^{*}_{S,t},\theta^*_{S,t}]}^2},
    && F_t = \ex{}{\norm{[w^{*}_{S,t},\theta^*_{S,t}] - [\widetilde{w}_{t},\widetilde{\theta}_{t}]}^2}, 
    && G_t = \frac{1}{2^t\lambda}\ex{}{\gapfunc^{(t)}\br{\widetilde{w}_{t},\widetilde{\theta}_{t}}}. \label{eq:EFG}
\end{align} }
Additionally, this will establish property \ref{prop:p1} because for any $t\in[T]$ it holds that, %
\begin{align}
    &\ex{}{\norm{[\bar{w}_t,\bar{\theta}_t] - [w_{t}^*,\theta_{t}^*]}^2} \nonumber \\
    &\leq  3\Bigg(\ex{}{\norm{[\bar{w}_t,\bar{\theta}_t] - [w^{*}_{S,t},\theta^*_{S,t}]}^2} + \ex{}{\norm{[w^{*}_{S,t},\theta^*_{S,t}] - [\widetilde{w}_{t},\widetilde{\theta}_{t}]}^2} + \ex{}{\norm{[\widetilde{w}_t,\widetilde{\theta}_t] - [w_{t}^*,\theta_{t}^*]}^2}\Bigg) \nonumber \\
    &\leq  3\Bigg(\underbrace{\ex{}{\norm{[\bar{w}_t,\bar{\theta}_t] - [w^{*}_{S,t},\theta^*_{S,t}]}^2}}_{E_t} + \underbrace{\ex{}{\norm{[w^{*}_{S,t},\theta^*_{S,t}] - [\widetilde{w}_{t},\widetilde{\theta}_{t}]}^2}}_{F_t} + \underbrace{\frac{1}{2^t\lambda}\ex{}{\gapfunc^{(t)}\br{\widetilde{w}_{t},\widetilde{\theta}_{t}}}}_{G_t}\Bigg). \label{eq:E-F-G}
\end{align}\
The second inequality comes from the strong convexity-strong concavity of the loss. 

\vspace{5pt}
\noindent\textit{Bounding $E_t$:}
We have that $E_t$ is bounded by the assumption made in the statement of Theorem \ref{thm:generalized-rr-convergence}.

\vspace{5pt}
\noindent\textit{Bounding $F_t$:}
\begin{align}
    \ex{}{\norm{[w^{*}_{S,t},\theta^*_{S,t}] - [\widetilde{w}_{t},\widetilde{\theta}_{t}]}^2}
    \leq \frac{\lip^2}{2^{2t} \lambda^2 n'} 
    \leq \frac{\rad^2\lip^2}{2304\cdot2^{2t} (\lip/\sqrt{n'})^2 n'}
    = \frac{\rad^2}{2304 \cdot 2^{2t}}. \label{eq:Ft-bound}
\end{align}
The first inequality comes from the stability of the regularized minimizer and Lemma \ref{lem:stability-implies-variance}. The second inequality comes from the setting of $\lambda \geq \frac{48\lip}{\rad\sqrt{n'}}$. 

\vspace{5pt}
\noindent\textit{Bounding $G_t$:} %
We have 
\begin{align*}
    \frac{1}{2^t\lambda}\,\,\ex{}{\gapfunc^{(t)}\br{\widetilde{w}_t,\widetilde{\theta}_t}} 
    &= \frac{1}{2^t\lambda}\,\,\ex{}{\ex{}{\gapfunc^{(t)}\br{\ex{}{w^{*}_{S,t} | \cF_{t-1}},\ex{}{\theta^*_{S,t} | \cF_{t-1}}} \Big| \cF_{t-1}}} \\
    &\leq  \frac{1}{2^{t}\lambda}\Big(\ex{}{\ex{}{\egapfunc^{(t)}\br{w^{*}_{S,t},\theta^*_{S,t}}\Big| \cF_{t-1}}}+\frac{\lip^2}{2^t\lambda n'}\Big) \\
    &= \frac{\lip^2}{2^{2t}\lambda^2 n'} 
    \leq \frac{\rad^2}{2304 \cdot 2^{2t}}. 
\end{align*}

The first equality comes from the definition of $[\widetilde{w}_t,\widetilde{\theta}_t]$. The first inequality comes from Lemma \ref{lem:weak-bounds-expected}, where we consider the algorithm stated in the lemma to be the algorithm which outputs the \emph{exact} regularized minimizer. Note this algorithm is $\frac{\lip^2}{2^t\lambda n'}$ stable. 
The second equality comes from the fact that $[w^{*}_{S,t},\theta^*_{S,t}]$ is the exact empirical saddle point. The final inequality uses the same analysis as in Eqn. \eqref{eq:Ft-bound}.

We thus have a final bound $3(E_t + F_t + G_t) \leq \frac{\rad^2}{2^{2t}}$.

\ifarxiv \pagebreak \fi
\paragraph{Property \ref{prop:p2}:}
Now assume $B_{t-1}$ holds. 
We have 
\ifarxiv \vfill \fi
{\small
\begin{align}
    \ex{}{\norm{[w^*_t,\theta^*_t]-[\bar{w}_{t-1},\bar{\theta}_{t-1}]}^2} 
    &\leq 2\ex{}{\norm{[w^*_{t},\theta^*_{t}] - [\widetilde{w}_{t-1},\widetilde{\theta}_{t-1}]}^2} + 2\ex{}{\norm{[\widetilde{w}_{t-1},\widetilde{w}_{t-1}] - [\bar{w}_{t-1},\bar{\theta}_{t-1}]}}^2 \nonumber \\
    &\leq 2\ex{}{\norm{[w^*_{t},\theta^*_{t}] - [\widetilde{w}_{t-1},\widetilde{\theta}_{t-1}]}^2} + 4E_{t-1} + 4F_{t-1}.\label{eq:proptwo-triangle-bound} %
\end{align}} \ifarxiv \vfill \fi 
\ifarxiv \noindent \fi Above $E_{t-1}$ and $F_{t-1}$ are as defined in \eqref{eq:EFG}.
We bound the remaining squared distance term in the following.  First, note that the primal function $F^{(t)}(\cdot,\theta_t^*)$ is strongly convex and $\forall w\in\cW$ it holds that
$\ip{\nabla_w F_{\cal D}^{(t)}(w_t^*,\theta_t^*)}{w_t^* - w} \leq 0$.  
Similar facts hold for %
$-F^{(t)}(w_t^*,\cdot)$. Thus we have
\ifarxiv \vfill \fi
\begin{align*}
    &\ex{}{\norm{[w^*_{t},\theta^*_{t}] - [\widetilde{w}_{t-1},\widetilde{\theta}_{t-1}]}^2} =\ex{}{\norm{\widetilde{w}_{t-1}-w^*_{t}}^2+ \|\theta^*_{t} -\widetilde{\theta}_{t-1}\|^2}  \\
    &\leq \ex{}{\frac{1}{2^t\lambda}\br{F_{\cD}^{(t)}(\widetilde{w}_{t-1},\theta_t^*) - F_{\cD}^{(t)}(w^*_{t},\theta_t^*) +  F_{\cD}^{(t)}(w^*_{t},\theta_t^*) - F_{\cD}^{(t)}(w_t^*, \widetilde{\theta}_{t-1})}} \\
    &= \mathbb{E}\Big[\frac{1}{2^t\lambda}\br{F_{\cD}^{(t-1)}(\widetilde{w}_{t-1},\theta_t^*) - F_{\cD}^{(t-1)}(w_t^*, \widetilde{\theta}_{t-1})} + \norm{\widetilde{w}_{t-1} - \bar{w}_{t-1}}^2 - \norm{\theta^*_t - \bar{\theta}_{t-1}}^2 \\
    &\quad -\norm{w^*_{t} - \bar{w}_{t-1}}^2 + \|\widetilde{\theta}_{t-1} - \bar{\theta}_{t-1}\|^2 \Big] \\
    &\leq \ex{}{\frac{1}{2^t\lambda}\br{F_{\cD}^{(t-1)}(\widetilde{w}_{t-1},\theta_t^*) - F_{\cD}^{(t-1)}(w_t^*, \widetilde{\theta}_{t-1})} + \norm{[\widetilde{w}_{t-1},\widetilde{\theta}_{t-1}] - [\bar{w}_{t-1},\bar{\theta}_{t-1}]}^2} \\
    &\leq \ex{}{\frac{1}{2^t\lambda}\br{F_{\cD}^{(t-1)}(\widetilde{w}_{t-1},\theta_t^*) - F_{\cD}^{(t-1)}(w_t^*, \widetilde{\theta}_{t-1})}} + 2E_{t-1} + 2F_{t-1} \\
    &\leq \ex{}{\frac{1}{2 \cdot 2^{t-1}\lambda}\br{\gapfunc^{(t-1)}(\widetilde{w}_{t-1},\widetilde{\theta}_{t-1})}} + 2E_{t-1} + 2F_{t-1} \\
    &\leq \frac{1}{2}G_{t-1} + 2E_{t-1} + 2F_{t-1}.
\end{align*} \ifarxiv \vfill \fi
\ifarxiv \noindent \fi The second inequality comes from removing the negative norm terms.
The third inequality comes from the definition of $E_{t-1}$ and $F_{t-1}$. %
The second to last inequality comes from the definition of $G_{t-1}$, as given in Eqn. \eqref{eq:EFG}. %
Plugging this result into \eqref{eq:proptwo-triangle-bound} and using the previously established bounds on $E_{t-1},F_{t-1},G_{t-1}$ (which hold under the assumed bound on $B_{t-1}$) we have %
$$\ex{}{\norm{[w^*_t,\theta^*_t]-[\bar{w}_{t-1},\bar{\theta}_{t-1}]}^2} \leq \frac{1}{2}G_{t-1} + 6 E_{t-1} + 6 F_{t-1} \leq \frac{\rad^2}{2^{2(t-1)}}.$$
\end{proof}

\ifarxiv \pagebreak \fi
We now turn to analyzing the utility of the algorithm to complete the proof. 
\begin{proof}[proof of Theorem \ref{thm:generalized-rr-convergence}]
Using the fact that $\gapfunc$ is $\sqrt{2}\lip$-Lipschitz and property \ref{prop:p1}, we have 
\begin{align}
    \ex{}{\gapfunc(\bar{w}_T,\bar{\theta}_T) - \gapfunc(w_{T}^*,\theta_{T}^*)} 
    &\leq \sqrt{2} \lip\ex{}{\norm{[\bar{w}_T,\bar{\theta}_T] - [w_{T}^*,\theta_{T}^*]}} \nonumber \\
    &\leq \frac{\sqrt{2}\rad\lip}{2^T}
    \leq \sqrt{2}\rad^2\lambda. \label{eq:error-to-final-min}
\end{align}

What remains is showing $\ex{}{\gapfunc(w_{T}^*,\theta_{T}^*)}$ is $\tilde{O}(\rad\alphad + \frac{\rad\lip}{\sqrt{n'}})$. Let $w' = \argmin\limits_{\theta\in\Theta}{F_{\cD}(w,\theta_T^*)}$ and $\theta' = \argmax\limits_{w\in\cW}{F_{\cD}(w_T^*,\theta})$. 
Using the fact that $F_{\cD}$ is convex-concave we have 
\ifarxiv \vfill \fi
\begin{align}
    \gapfunc(w^*_T,\theta^*_T) = F_{\cD}(w_T^*,\theta') - F_{\cD}(w',\theta_T^*)
    &\leq \ip{G_{\cD}(w_T^*,\theta_T^*)}{[w^*_T,\theta^*_T] - [w',\theta']} \label{eq:gap-grad-bound} 
\end{align}
\ifarxiv \vfill \noindent \fi where $G_{\cD}$ is the population loss saddle operator.
Further by the definition of $F^{(T)}$ and denoting $G_{\cD}^{(T)}$ as the saddle operator for $F_{\cD}^{(T)}$ we have
\ifarxiv \vfill \fi
\begin{align*}
    G_{\cD}(w_{T}^*,\theta_{T}^*)
    &= G_{\cD}^{(T)}(w_{T}^*,\theta_{T}^*) - 2\lambda \sum\limits_{t=0}^{T-1} 2^{t+1}([w_{T}^*,-\theta_{T}^*] - [\bar{w}_t,-\bar{\theta}_t]) 
\end{align*}
\ifarxiv \vfill \noindent \fi
Thus plugging the above into Eqn. \eqref{eq:gap-grad-bound} we have
\ifarxiv \vfill \fi
\begin{align*}
   \gapfunc(w^*_T,\theta^*_T) &\leq  
   \ip{G_{\cD}^{(T)}(w_{T}^*,\theta_{T}^*)}{[w^*_T,\theta^*_T] - [w',\theta']} \\
   &\textstyle\quad- \ip{2\lambda \sum\limits_{t=0}^{T-1} 2^{t+1}([w_{T}^*,-\theta_{T}^*] - [\bar{w}_t,-\bar{\theta}_t])}{[w^*_T,\theta^*_T] - [w',\theta']} \\
   &\textstyle\leq - \ip{2\lambda \sum\limits_{t=0}^{T-1} 2^{t+1}([w_{T}^*,-\theta_{T}^*] - [\bar{w}_t,-\bar{\theta}_t])}{[w^*_T,\theta^*_T] - [w',\theta']} \\
   &\textstyle\leq 2\rad \lambda \sum\limits_{t=0}^{T-1} 2^{t+1}\norm{[w_{T}^*,-\theta_{T}^*] - [\bar{w}_t,-\bar{\theta}_t]} \\
   &\textstyle= 2\rad \lambda \sum\limits_{t=0}^{T-1} 2^{t+1}\norm{[w_{T}^*,\theta_{T}^*] - [\bar{w}_t,\bar{\theta}_t]}.
\end{align*}
\ifarxiv \vfill \noindent \fi
Above, the second inequality comes from the first order optimally conditions for $[w_T^*,\theta_T^*]$, the third from Cauchy Schwartz and a triangle inequality. The final equality uses the definition of the Euclidean norm and the fact that for any $a,b\in\re$, $(-a - (-b))^2 = (a-b)^2$.

Taking the expectation on both sides of the above we have the following derivation,
{\small
\begin{align}
    &\ex{}{\gapfunc(w_{T}^*,\theta_{T}^*)} 
    \leq 2\rad\ex{}{\lambda \sum\limits_{t=0}^{T-1} 2^{t+1}\norm{[w_{T}^*,\theta_{T}^*] - [\bar{w}_t,\bar{\theta}_t]}} \nonumber \\
    &\overset{(i)}{\leq} 4\rad\ex{}{\lambda \sum\limits_{t=0}^{T-1} 2^{t} \br{\norm{[w_{t+1}^*,\theta_{t+1}^*] - [\bar{w}_t,\bar{\theta}_t]} + \sum\limits_{r=t+1}^{T-1}\norm{[w_{r+1}^*,\theta_{r+1}^*]  - [w_{r}^*,\theta_{r}^*]}}} \nonumber \\
    &\leq 4\rad\ex{}{\lambda \sum\limits_{t=0}^{T-1} 2^{t} \br{\norm{[w_{t+1}^*,\theta_{t+1}^*] - [\bar{w}_t,\bar{\theta}_t]} + \sum\limits_{r=t+1}^{T-1}\norm{[w_{r+1}^*,\theta_{r+1}^*] - [\bar{w}_{r},\bar{\theta}_{r}]} + \norm{[\bar{w}_{r},\bar{\theta}_{r}] - [w_{r}^*,\theta_{r}^*]}}} \nonumber \\
    &= 4\rad\ex{}{\lambda \sum\limits_{t=0}^{T-1} 2^{t} \norm{[w_{t+1}^*,\theta_{t+1}^*] - [\bar{w}_t,\bar{\theta}_t]} + \lambda\sum\limits_{t=0}^{T-1} 2^{t}\sum\limits_{r=t+1}^{T-1} \br{\norm{[w_{r+1}^*,\theta_{r+1}^*] - [\bar{w}_{r},\bar{\theta}_{r}]} + \norm{[\bar{w}_{r},\bar{\theta}_{r}] - [w_{r}^*,\theta_{r}^*]}}} \nonumber \\
    &\overset{(ii)}{=} 4\rad\ex{}{\lambda \sum\limits_{t=0}^{T-1} 2^{t} \norm{[w_{t+1}^*,\theta_{t+1}^*] - [\bar{w}_t,\bar{\theta}_t]} + \lambda\sum\limits_{r=1}^{T-1}\sum\limits_{t=0}^{r-1} 2^{t} \br{\norm{[w_{r+1}^*,\theta_{r+1}^*] - [\bar{w}_{r},\bar{\theta}_{r}]} + \norm{[\bar{w}_{r},\bar{\theta}_{r}] - [w_{r}^*,\theta_{r}^*]}}} \nonumber \\
    &= 4\rad\ex{}{\lambda \sum\limits_{t=0}^{T-1} 2^{t} \norm{[w_{t+1}^*,\theta_{t+1}^*] - [\bar{w}_t,\bar{\theta}_t]} + \lambda\sum\limits_{r=1}^{T-1}\br{\norm{[w_{r+1}^*,\theta_{r+1}^*] - [\bar{w}_{r},\bar{\theta}_{r}]} + \norm{[\bar{w}_{r},\bar{\theta}_{r}] - [w_{r}^*,\theta_{r}^*]}}\sum\limits_{t=0}^{r-1} 2^{t}} \nonumber \\
    &\overset{(iii)}{\leq} 4\rad\br{\lambda \sum\limits_{t=0}^{T-1} 2^{t}\br{\frac{\rad}{2^{t}}} + \lambda\sum\limits_{r=1}^{T-1}\br{\frac{2\rad
    }{2^{r}}}\sum\limits_{t=0}^{r-1} 2^{t}} \nonumber \\
    &\leq 4\rad\br{ \lambda \sum\limits_{t=0}^{T-1} 2^{t}\br{\frac{\rad}{2^{t}}} + \lambda\sum\limits_{r=1}^{T-1}\br{\frac{\rad}{2^{r-1}}  }2\cdot 2^{r-1}} \nonumber \\
    &= 4\lambda \sum\limits_{t=0}^{T-1} \rad^2 + 8\lambda\sum\limits_{r=1}^{T-1} \rad^2 \nonumber \\
    &\leq 12T \lambda \rad^2 
    \label{eq:final-gap-bound}
\end{align}}
Above, $(i)$ and the following inequality both come from the triangle inequality. Equality $(ii)$ is obtained by rearranging the sums. Inequality $(iii)$ %
comes from applying properties \ref{prop:p1} and \ref{prop:p2} proved above. The last equality comes from the setting of $\lambda$ and $T$.

Now using this result in conjunction with Eqn. \eqref{eq:error-to-final-min} 
we have
\begin{align*}
    \gap(\cR) = \sqrt{2}\lambda\rad^2 + 12T\lambda\rad^2 = O\br{\log(n)\rad^2\lambda}.
\end{align*}
Above we use the fact that $T=\log(\frac{\lip}{\rad\lambda})$ and $\lambda \geq \frac{\lip}{\rad\sqrt{n'}}$, and thus $T=O(\log(n))$.
\end{proof}

Finally, we prove Theorem \ref{thm:nonsmooth-minimax-alg} leveraging the relative accuracy assumption. 

\begin{proof}[Proof of Theorem \ref{thm:nonsmooth-minimax-alg}]
First, observe that under the setting of $\lambda=\frac{48}{\rad}\br{\alphad + \frac{\lip}{\sqrt{n'}}}$ used in the theorem statement that $\log(n)\rad^2\lambda = O\br{\log(n)\rad\alphad + \frac{\log^{3/2}(n)\rad\lip}{\sqrt{n}}}$. Thus what remains is to show that the distance condition required by Theorem \ref{thm:generalized-rr-convergence} holds. That is, we now show that if $\weakalg$ satisfies $\alphad$-relative accuracy, then 
for all $t\in [T]$ it holds that $\ex{}{\norm{[\bar{w}_t,\bar{\theta}_t] - [w^{*}_{S,t},\theta^*_{S,t}]}^2} \leq \frac{\rad^2}{12\cdot2^{2t}}$.

To prove this property, we must leverage the induction argument made by Lemma \ref{lem:phase-distance-bound}. Specifically, to prove the condition holds for some $t\in[T]$, assume 
$\rad_t^2 = \ex{}{\norm{[w_{t}^*,\theta_{t}^*] - [\bar{w}_{t-1},\bar{\theta}_{t-1}]}}^2 \leq \frac{\rad^2}{2^{2(t-1)}}$ (recall the base case for $t=1$ trivially holds). As shown in the proof of Lemma \ref{lem:phase-distance-bound}, this implies that the quantities
$F_t,G_t$ (as defined in \ref{eq:EFG}) are bounded by $\frac{\rad^2}{2304 \cdot 2^{2t}}$. We thus have
\begin{align} \label{eq:Et-bound-relative-acc}
    \ex{}{\norm{[\bar{w}_t,\bar{\theta}_t] - [w^{*}_{S,t},\theta^*_{S,t}]}^2}
    &\overset{(i)}{\leq} \frac{\ex{}{F_S^{(t)}(\bar{w}_t,\theta^*_{S,t}) - F_S^{(t)}(w^*_{S,t},\bar{\theta}_t)}}{2^{t}\lambda} \nonumber \\
    &\overset{(ii)}{\leq} \frac{\alphad\ex{}{\|[w^{*}_{S,t},\theta^*_{S,t}]-[\bar{w}_{t-1},\bar{\theta}_{t-1}]\|}}{2^{t}\lambda} \nonumber \\
    &\leq \frac{\alphad\ex{}{\|[w^{*}_{S,t},\theta^*_{S,t}]-[w^*_t,\theta^*_t]\|+\|[w^*_t,\theta^*_t]-[\bar{w}_{t-1},\bar{\theta}_{t-1}]\|}}{2^{t}\lambda} \nonumber \\
    &\overset{(iii)}{\leq} \frac{(\sqrt{F_t}+\sqrt{G_t}+\rad_t)\alphad}{2^{t}\lambda} 
    \overset{(iv)}{\leq} \frac{2\rad\alphad}{2^{t}2^{t-1}\lambda} 
    \leq \frac{\rad^2}{12 \cdot 2^{2t}},
\end{align}
where $B_t$ is as defined in property \ref{prop:p2}.
Inequality $(i)$ comes from Lemma \ref{lem:sc-sc-distance}.
Inequality $(ii)$ comes from the $\alphad$-relative accuracy assumption on $\weakalg$, and the fact that each $f^{(t)}$ is $2\lip$-Lipschitz. That is, observe 

\begin{align*}    \max\limits_{w,\theta\in\cW\times\Theta}\norm{\nabla f^{(t)}(w,\theta,x)} \leq \lip + 2\sum_{k=0}^{t-1}\rad 2^{k+1}\lambda\leq \lip + 4\rad 2^{T}\lambda \leq 5\lip
\end{align*}

Inequality $(iii)$ comes from a triangle inequality and the definition of $F_t,G_t$ and $B_t$. Inequality $(iv)$ comes from the induction hypothesis (specifically property \ref{prop:p2}) and the bounds on $F_t$ and $G_t$ established above.  
The last inequality in Eqn. \eqref{eq:Et-bound-relative-acc} comes from the setting $\lambda \geq 48\alphad/\rad$.
\end{proof}

%% file: COLT2023/Sections/app-dp-rates.tex
\section{Missing Results from Section \ref{sec:dp-rates}} 
\label{app:dp-rates}

\subsection{Stochastic Gradient Descent Ascent (SGDA)} \label{app:sgda}

Let $F:\cW\times\Theta\mapsto\re$ have saddle operator $G:\cW\times\Theta\mapsto\re^d$ and associated strong gap $\gap^F$. %
We define the SGDA algorithm in the following manner. Let $T,\eta \geq 0$. Let $[w_0,\theta_0]$ be any vector in $\cW\times\Theta$. SGDA uses the following update rule. For $t\in[T-1]$ let $\nabla_t$ be a random vector (which may depend on $\nabla_1,...,\nabla_{t-1}$ and $[w_0,\theta_0],...,[w_{t-1},\theta_{t-1}]$) that is a unbiased estimate of $G(w_{t-1},\theta_{t-1})$ conditional on $[w_{t-1},\theta_{t-1}]$ and has bounded variance. We define 
\begin{equation}
    [w_t,\theta_t] = \Pi_{\cW\times \Theta}\br{[w_{t-1},\theta_{t-1}] - \eta \nabla_t}, \quad t\in[T-1]
\end{equation}
where $\Pi_{\cW \times \Theta}$ is the orthogonal projection onto $\cW\times\Theta$.
The output of SGDA is defined to be
\begin{equation}
    [\bar{w},\bar{\theta}] = \frac{1}{T}\sum_{t=0}^{T-1}  [w_t,\theta_t].
\end{equation}
We have the following result for the convergence of SGDA. %
\begin{lemma} \label{lem:noisy_SGDA}
Assume $\forall t\in[T-1]$ that $\ex{}{\nabla_t}=G(w_t,\theta_t)$ and $\ex{}{\norm{\nabla_t - G(w_t,\theta_t)}^2} \leq \tau^2$, then the algorithm, $\cA$, that is SGDA run with parameters $T,\eta > 0$ satisfies for any $w\in\cW$ and $\theta\in\Theta$,
\begin{align*}
    \ex{}{F(\bar{w},\theta) - F(w,\bar{\theta})} \leq \frac{\norm{[w_0,\theta_0]-[w,\theta]}^2}{2\eta T} + \frac{\eta}{2}\br{\lip^2+\tau^2}
\end{align*}
\end{lemma}
This result is somewhat implicit in \citet[Lemma 3]{yang-dp-sgda}, but for completeness we provide a short proof here.
\begin{proof}
By the convexity-concavity of $F$ we have for any $[w,\theta]\in \cW\times\Theta$ that
\begin{align*}
    F(w_t, \theta) - F(w,\theta_t) &\leq \ip{G(w_t,\theta_t)}{[w_t,\theta_t]-[w,\theta]}
\end{align*}
and thus taking the expectation (conditional on $[w_t,\theta_t]$) and using the fact that each $\nabla_t$ is unbiased we have
\begin{align*}
    \ex{}{F(w_t, \theta) - F(w,\theta_t)} &\leq \ip{\ex{}{\nabla_t}}{[w_t,\theta_t]-[w,\theta]}.
\end{align*}
Using $2\ip{a}{b}=\|a\|^2 + \|b\|^2 - \|a-b\|^2$ and the fact that the projection is nonexpansive, %
we have
\begin{align*}
    &\ex{}{F(w_t, \theta) - F(w,\theta_t)} \\
    &\leq \ex{}{\frac{1}{2\eta}\br{\norm{[w_t,\theta_t]-[w,\theta]}^2-\norm{[w_{t+1},\theta_{t+1}]-[w,\theta]}^2} + \frac{\eta}{2}\norm{\nabla_t}^2} \\
    &= \ex{}{\frac{1}{2\eta}\br{\norm{[w_t,\theta_t]-[w,\theta]}^2-\norm{[w_{t+1},\theta_{t+1}]-[w,\theta]}^2} +\frac{\eta}{2}\br{\norm{G(w_t,\theta_t)}^2 + \norm{G(w_t,\theta_t) - \nabla_t}^2}} \\
    &\leq \ex{}{\frac{1}{2\eta}\br{\norm{[w_t,\theta_t]-[w,\theta]}^2-\norm{[w_{t+1},\theta_{t+1}]-[w,\theta]}^2}} + \frac{\eta}{2}\br{\lip^2 + \tau^2},
\end{align*}
where in the first equality we use that $\mathbb{E}[\langle G(w_t,\theta_t),G(w_t,\theta_t)-\nabla_t \rangle]=0$, due to the unbiasedness of the stochastic oracle.

Summing over all $T$ iterations and taking the average we obtain for the average iterate, $\bar{w},\bar{\theta}$, and any $[w,\theta]\in\cW\times\Theta$ that
\begin{align*}
   \ex{}{F\Big(\frac{1}{T}\sum_{t=0}^{T-1} w_t, \theta \Big) - F\Big(w,\frac{1}{T}\sum_{s=1}^T \theta_t\Big)} 
    &\leq \ex{}{\frac{1}{T}\sum_{t=0}^{T-1} [F(w_t, \theta) - F(w,\theta_t)]} \\
    &\leq \frac{\norm{[w_0,\theta_0]-[w,\theta]}^2}{2\eta T} + \frac{\eta}{2}\br{\lip^2+\tau^2}
\end{align*}
\end{proof}

\subsection{Private algorithm for the empirical gap (Noisy SGDA)} \label{app:noisy-sgda}

We here provide an implementation of SGDA (see Appendix \ref{app:sgda} above) which is differentially private and yields convergence guarantees for the empirical gap.
Let $M_1,...,M_T$ each be a batch of $m=\max\bc{n\sqrt{\frac{\epsilon}{4T}},1}$ samples, each sampled uniformly with replacement from $S$. Let 
$\sigma^2=\frac{c_0 T\lip^2\log(1/\delta)}{n^2\epsilon^2}$ for some universal constant $c_0$ and $\xi_1,\dots,\xi_T$ each be sampled i.i.d. from $\cN(0,\mathbb{I}_d\sigma^2)$.
We define 
\begin{align*}
    \nabla_t = \frac{1}{m}\sum_{x\in M_{t}}g(w_{t-1},\theta_{t-1};x)+\xi_{t}.
\end{align*}
Notice that $\nabla_t$ as defined above satisfies the assumptions for Lemma \ref{lem:noisy_SGDA} with respect to the empirical saddle operator, $G_S$, for some finite $\tau$.

We have the following result for SGDA run with this stochastic oracle.
\begin{theorem}
Let $[w,\theta]\in\cW\times\Theta$ such that $\ex{}{\norm{[w_0,\theta_0]-[w,\theta]}} \leq \hat{\erad}$.
Let $\cA$ be the algorithm SGDA run with $\nabla_1,\dots,\nabla_T$ as described above, 
$T = \min\bc{\frac{n}{8},\frac{n^2\epsilon^2}{32d\log(1/\delta)}}$, and 
$\eta=\frac{\hat{\erad}}{L\sqrt{T}}$
. Algorithm $\cA$ is $(\epsilon,\delta)$-DP, has gradient complexity $O\br{\min\bc{\frac{n^2\epsilon^{1.5}}{\sqrt{d\log(1/\delta)}}, n^{3/2}}}$, and satisfies %
\begin{equation*}
    \ex{}{F_S(\bar{w},\theta) - F_S(w,\bar{\theta})} = O\br{\frac{\hat{\erad}\lip\sqrt{d\log(1/\delta)}}{n\epsilon} + \frac{\hat{\erad}\lip}{\sqrt{n}}}.
\end{equation*}
\end{theorem}
The proof of the utility guarantee follows directly from applying Lemma \ref{lem:noisy_SGDA} with %
$\tau=O(\lip + \sqrt{d}\sigma)=O(\lip)$. 
The proof of the privacy guarantee relies on the moments accountant analysis, for which we provide the following restatement. %
\begin{theorem}[\label{thm:moment-accountant}\cite{Abadi16,KLL21}]
Let $\epsilon,\delta \in (0,1]$ and $c$ be a universal constant. Let $D\in\cY^n$ be a dataset over some domain $\cY$, and let $h_1,...,h_T:\cY\mapsto\re^d$ be a series of (possibly adaptive) queries such that for any $y\in\cY$, $t\in[T]$,  $\norm{h_t(y)}_2 \leq \lip$. Let $\sigma \geq \frac{c \lip \sqrt{T\log(1/\delta)}}{n\epsilon}$ and %
$T \geq \frac{n^2\epsilon}{b^2}$.
Then the algorithm which samples batches of size $B_1,..,B_t$ of size $b$ uniformly at random and outputs $\frac{1}{b}\sum_{y\in B_t}h_t(y) + g_t$ for all $t\in[T]$ where $g_t \sim \cN(0,\mathbb{I_d}\sigma^2)$, is $(\epsilon,\delta)$-DP.
\end{theorem}
It can be verified for the described noisy SGDA implementation that
$\sigma \geq \frac{c_1 \lip \sqrt{T\log(1/\delta)}}{n\epsilon}$ 
and $T \geq \frac{n^2\epsilon}{m^2}$ 
and thus the algorithm is $(\epsilon,\delta)$-DP.

%% file: COLT2023/Sections/app-strong-vs-weak.tex
\section{Missing Result from Section \ref{sec:strong-vs-weak}}
\label{app:strong-vs-weak}

\subsection{Low variance and weak gap implies strong gap} 
\label{app:var-implies-strong-gap-proof}
\begin{proof}[proof of Proposition \ref{prop:var-implies-strong-gap}]
Consider the virtual algorithm, $\cB(\cA,\cD) = \ex{\hat{S}\sim\cD^n,\cA}{\cA(\hat S)} = [\widetilde{w},\widetilde{\theta}]$. Note this algorithm is deterministic and does not depend on any specific dataset drawn from $\cD$. We first show that gap function at the output of $\cB$ is bounded by the weak gap of $\cA$. We have
\begin{align}
    \gapfunc(\cB(\cA,\cD)) &= \max_{\theta\in\Theta}\bc{F_{\cD}(\cB_w(\cA,\cD),\theta)} - \min_{w\in\cW}\bc{{F_{\cD}(w,\cB_\theta(\cA,\cD))}} \nonumber \\
    &= \max_{\theta\in\Theta}\bc{F_{\cD}\Big(\ex{\hat{S}\sim\cD^n,\cA_w}{\cA_w(\hat S)},\theta\Big)} - \min_{w\in\cW}\bc{{F_{\cD}\Big(w,\ex{\hat{S}\sim\cD^n,\cA_\theta}{\cA_\theta(\hat S)}\Big)}} \nonumber \\
    &\leq \max_{\theta\in\Theta}\bc{\ex{\hat{S}\sim\cD^n,\cA_w}{F_{\cD}(\cA_w(\hat S),\theta)}} - \min_{w\in\cW}\bc{\ex{\hat{S}\sim\cD^n,\cA_\theta}{{F_{\cD}(w,\cA_\theta(\hat S))}}} \nonumber \\
    &= \weakgap(\cA), \label{eq:strong-weak-bound}
\end{align}
where the second equality follows from the definition of $\cB$ and the inequality follows from Jensen's inequality.

Now by the assumption that $\cA$ is low variance, we have
\begin{equation} \label{eq:var-to-dist-bound}
    \ex{\cA,S}{\norm{\cA(S) - \cB(\cA,\cD)}^2} = \ex{\cA,S}{\norm{\cA(S) - \ex{\hat{S}\sim\cD^n,\cA}{\cA(\hat S)}}^2} \leq \tau^2.
\end{equation}
\end{proof}

Thus using the Lipschitzness of $\gapfunc$ we obtain%
\begin{align*}
    \gap(\cA) - \weakgap(\cA) &= \ex{S,\cA}{\widehat{\gap}(\cA_w(S),\cA_\theta(S))} - \weakgap(\cA) \\
    &\leq \ex{S,\cA}{\widehat{\gap}(\cA_w(S),\cA_\theta(S))} - \widehat{\gap}(\cB(\cA,\cD)) \\
    &\leq \lip\ex{S,\cA}{\norm{\cA(S)-\cB(\cA,\cD)}} 
    \leq \lip\tau.
\end{align*}
The first inequality comes from Eqn. \eqref{eq:strong-weak-bound}. The second inequality comes from the Lipschitzness of the gap function. The third inequality comes from Eqn. \eqref{eq:var-to-dist-bound}.
Thus we ultimately have
\begin{equation}
    \gap(\cA) \leq \weakgap(\cA) + \lip\tau.
\end{equation}

\subsection{Stability-Risk Tradeoff} \label{app:stability-risk-tradeoff}
\begin{proof}[proof of Theorem \ref{thm:stab-risk-tradeoff}]
Let $f(w;x)=\ip{w}{x}$. Let $0<K<\min\bc{n,d}$ be a parameter to be chosen later and define $U=\bc{\pm1}^K$. For any $\bssigma\in U$ define
$S_{\bssigma}=\bc{\lip\bssigma_1 e_1,...,\lip\bssigma_K e_K,0,...,0}$, where $e_j$ is the $j$'th standard basis vector. We will denote
$F(w;S_{\bssigma}) = \frac{1}{n}\sum_{x\in S_{\bssigma}}f(w;x)$.
Note that 
$$w_{\bssigma}^* = \argmin\limits_{w\in\cW}\bc{F(w;S_{\bssigma})} = \frac{B}{\sqrt{K}}\sum_{j\in[K]}-\bssigma_j e_j.$$ 
Further, for any $\bssigma\in U$, $F(w_{\bssigma}^*;S_{\bssigma})=-\frac{BL\sqrt{K}}{n}$. 

By Yao's minimax principle, it suffices to consider deterministic algorithms and lower bound the expected risk w.r.t. some distribution over the packing. Considering the uniform distribution over the packing and setting $K=\frac{\rad^2}{\Delta^2}$ we have

\begin{align*}
    \ex{\bssigma\sim \mbox{\footnotesize Unif}(U)}{F(\cA(S_{\bssigma});S_{\bssigma}) - F(w^*_{\bssigma};S_{\bssigma})} &= \frac{1}{|U|}\sum_{\bssigma\in U} F(\cA(S_{\bssigma});S_{\bssigma}) + \frac{BL\sqrt{K}}{n} \\
    &\overset{(i)}{=} \frac{1}{|U|}\sum_{\bssigma\in U} \bs{\frac{1}{n}\sum_{j\in[K]}\lip\bssigma_j\cA(S_{\bssigma})_j + \frac{1}{n}\sum_{j\in[K]}\frac{BL}{\sqrt{K}}} \\
    &= \frac{1}{n|U|}\sum_{j\in[K]}\sum_{\bssigma\in U} \lip\bssigma_j\cA(S_{\bssigma})_j + \frac{BL}{\sqrt{K}} \\
    &= \frac{1}{n|U|}\sum_{j\in[K]}\sum_{\bssigma\in U:\bssigma_j=1} \lip\br{\cA(S_{\bssigma})_j - \cA(S_{\bssigma_{-j}})_j} + \frac{2BL}{\sqrt{K}} \\
    &\overset{(ii)}{\geq} \frac{1}{n|U|}\sum_{j\in[K]}\sum_{\bssigma\in U:\bssigma_j=1} -\lip\Delta + \frac{2BL}{\sqrt{K}} \\
    &= \frac{1}{n|U|}\sum_{j\in[K]}\sum_{\bssigma\in U:\bssigma_j=1} \frac{BL}{\sqrt{K}} \\
    &= \frac{BL\sqrt{K}}{2n}
\end{align*}
where $(i)$ comes from the definition of the loss function and the fact that the dataset consists of $K$ standard basis vectors (up to sign) and $n-K$ zero vectors and $(ii)$ comes from the $\Delta=\frac{B}{\sqrt{K}}$ stability property of $\cA$ (i.e. $\cA(S_{\bssigma_{-j}})_j - \cA(S_{\bssigma})_j \leq \Delta \implies \cA(S_{\bssigma})_j - \cA(S_{\bssigma_{-j}})_j \geq -\Delta$). 
Finally, note that by the setting of $K$ that $\frac{\rad \lip\sqrt{K}}{n} = \frac{\rad^2\lip}{\Delta n}$.
\end{proof}